\documentclass{article}

 \usepackage{arxiv}

\usepackage{booktabs}       
\usepackage{amsfonts}       
\usepackage{amsmath} 
\usepackage{dsfont}
\usepackage{amssymb}
\usepackage{mathtools}
\usepackage{amsthm}
\usepackage{graphicx}
\usepackage{subcaption}
\usepackage{nicefrac}
\usepackage{mathrsfs} 
\usepackage{nicematrix}
\usepackage{float}

\usepackage{xr}
\usepackage{tikz}
\usepackage{pgfplots}

\usepackage{adjustbox}
\usepackage{bm, braket}

\usepackage{doi}
\usepackage{hyperref}

\usepackage{extarrows} 
\usepackage{enumitem}
\usepackage[linesnumbered, ruled, vlined]{algorithm2e}

\usepackage{wrapfig}

\SetKwComment{tcp}{$\bullet$ }{}
\SetKw{Sample}{Sample}
\SetKw{Define}{Define}

\makeatletter
\DeclareFontEncoding{LS1}{}{}
\DeclareFontSubstitution{LS1}{stix}{m}{n}
\DeclareMathAlphabet{\mathscr}{LS1}{stixscr}{m}{n}
\makeatother
\DeclareMathOperator*{\argmax}{arg\,max}

\newcommand{\bSigma}{\bm{\Sigma}}

\newcommand{\bepsilon}{\bm{\epsilon}}

\newcommand{\bmu}{\bm{\mu}}

\newcommand{\bA}{\mathbf{A}}

\newcommand{\bC}{\mathbf{C}}

\newcommand{\bG}{\mathbf{G}}

\newcommand{\bI}{\mathbf{I}}

\newcommand{\bS}{\mathbf{S}}

\newcommand{\bW}{\mathbf{W}}
\newcommand{\bX}{\mathbf{X}}
\newcommand{\bY}{\mathbf{Y}}

\newcommand{\be}{\mathbf{e}}

\newcommand{\bg}{\mathbf{g}}

\newcommand{\bv}{\mathbf{v}}
\newcommand{\bw}{\mathbf{w}}
\newcommand{\bx}{\mathbf{x}}
\newcommand{\by}{\mathbf{y}}
\newcommand{\bz}{\mathbf{z}}

\newcommand{\hbx}{\hat{\mathbf{x}}}

\newcommand{\bbE}{\mathbb{E}}
\newcommand{\bbR}{\mathbb{R}}

\newcommand{\bbP}{\mathbb{P}}
\newcommand{\bbN}{\mathbb{N}}

\newcommand{\calN}{\mathcal{N}}

\newcommand{\calR}{\mathcal{R}}

\newcommand{\tr}{\text{Tr}}

\newcommand{\rarrowp}{\xrightarrow[]{\bbP}}

\theoremstyle{plain}
\newtheorem{remark}{Remark}
\newtheorem{emp}{Empirical Observation}

\newtheorem{assumption}{Assumption}
\newtheorem{proposition}{Proposition}

\newtheorem{theorem}{Theorem}

\newtheorem{definition}{Definition}
\newtheorem{lemma}{Lemma}

\title{Gaussian Universality for Diffusion Models}

\author{ Reza Ghane\thanks{Equal Contribution} \ \thanks{Department of Electrical Engineering, 
	California Institute of Technology} \And Anthony Bao\footnotemark[1] \ \thanks{Department of Electrical and Computer Engineering, University of Texas at Austin}\And  Danil Akhtiamov\footnotemark[1] \ \thanks{Department of Computing and Mathematical Sciences, 
	California Institute of Technology} \And Babak Hassibi \footnotemark[2] \ \footnotemark[4] }

\date{\today}

\begin{document}

\maketitle

\begin{abstract}
We investigate Gaussian Universality  for data distributions generated via diffusion models. By Gaussian Universality we mean that the test error of  a generalized linear model $f(\bW)$ trained for  a classification task on the diffusion data matches the test error of $f(\bW)$ trained on the Gaussian Mixture with matching means and covariances per class. 
In other words, the test error depends only on the first and second order statistics of the diffusion-generated data in the linear setting. 
As a corollary, the analysis of the test error for linear classifiers can be reduced to Gaussian data from diffusion-generated data.  Analysing the performance of models trained on synthetic data is a pertinent problem due to the surge of methods such as \cite{sehwag2024stretchingdollardiffusiontraining}.  Moreover, we show that, for any $1$- Lipschitz scalar function $\phi$, $\phi(\bx)$ is close to $\bbE \phi(\bx)$ with high probability for $\bx$ sampled from the conditional diffusion model corresponding to each class.  Finally, we note that current approaches for proving universality do not apply to diffusion-generated data as the covariance matrices of the data tend to have vanishing minimum singular values, contrary to the assumption made in the literature. This leaves extending previous mathematical universality results as an intriguing open question.

\end{abstract}

\section{Introduction}
A remarkable contribution of deep learning is the advent of generative models for image and video generation. Diffusion-based generative models \cite{sohl-dickstein_2015}, \cite{song2019generative},\cite{ddpm},\cite{song2020denoising},\cite{dhariwal2021diffusion},\cite{song2021scorebased},\cite{kingma2021variational}, \cite{karras2022elucidating}, in particular, have enjoyed tremendous success in vision [LDM \cite{rombach2022high}, audio [Diffwave \cite{kong2020diffwave}] and text generation [D3PM \cite{austin2021structured}]. For an overview of diffusion models and their applications, we refer to the surveys \cite{croitoru2023diffusion} and \cite{yang2023diffusion}. 

Despite significant progress in training methods, network architecture design, and hyperparameter tuning, there has been relatively little work done on understanding rigorous mathematical properties of the data generated by diffusion models. Through theory and experiments, we argue that images generated by conventional diffusion models satisfy a form of Gaussian Universality. 

Gaussian Universality is an overloaded term. It often means that certain characteristics of data coming from $k$ classes distributed as $\bbP_1,\dots, \bbP_k$ with means $\bmu_1,\dots, \bmu_k$ and covariance matrices $\bSigma_1,\dots, \bSigma_k$ remain unchanged when replaced by data coming from the corresponding Gaussians $\mathcal{N}(\bmu_1,\bSigma_1), \dots, \mathcal{N}(\bmu_k,\bSigma_k)$. Informally, this means that a data representation satisfying Gaussian Universality behaves similarly to a Gaussian Mixture Model with matching means and covariances per class.

Gaussian Universality has been observed in the empirical distribution of eigenvalues of Gram matrices constructed from real-world data \cite{seddik2020random, levi2023underlying}. While this phenomenon is compelling, we focus on {\it Gaussian Universality of the test error}, which we consider a key characteristic in machine learning. Specifically, we compare the generalization error of generalized linear models trained on diffusion-generated data to those trained on Gaussian Mixtures with matching means and covariances, observing a close match.

As a technical step towards elucidating this phenomenon, we argue that, when the reverse process is a contraction, one can establish an Approximate Concentration of Measure phenomenon for the distribution of the output, which is a common assumption in theoretical works on Gaussian Universality of the test error.

While there are many aspects to building a diffusion model for data synthesis such as training the denoiser and choosing the forward process and noise schedule, in this work we take a higher-level approach and mainly focus on the sampling process of a pre-trained diffusion model. Our arguments are agnostic to the training procedure and the denoiser's architectural details.

We believe that the present study is important both for advancing our theoretical understanding of generative models for images and their limitations, as well as the role of data in supervised ML:
\begin{itemize}
    \item We show that distributions that can be generated via diffusion models satisfy an Approximate Concentration of Measure Property \ref{def: ACoM}. Similar properties are required in most apporaches to proving universality.  
    
    \item As discussed in \cite{goldblum2023perspectives}  and \cite{nakkiran2021towards}, one of the reasons we do not have a practical theory of deep learning is the lack of clean mathematical models that describe real-world data. In view of the universality results outlined in the presented work, combined with the empirical observations showing that diffusion-generated distributions can approximate real-world data well, we suggest that it suffices to consider GMMs. The analysis of performance of models trained on GMMs has been a topic of active research recently; see, e.g. \cite{thrampoulidis2020theoretical, loureiro2021learning}.   
\end{itemize}

Finally, we note that while this paper primarily assesses Gaussian Universality of test error for diffusion models trained on real-world data, the theoretical section is self-contained and presents all necessary mathematical formulations, results, and references.

\section{Related Works}
The theoretical analysis of diffusion models and the images generated by such models remains an underexplored area.
\begin{itemize}
    \item In an emerging line of work, many papers have analyzed the output distributions and convergence of diffusion models through the lens of Langevin dynamics. \cite{chen2022sampling} show that denoising diffusion probabilistic models (DDPM) and critically damped Langevin Dynamics (CLD) can efficiently sample from any arbitrary distribution, assuming accurate score estimates - an assumption central to many works in this area. While among the first works to provide quantitative polynomial bounds on convergence, the high-dimensional nature of the problem means that estimating the score function may be practically impossible. Furthermore, it is infeasible to verify the validity of this assumption, as we do not have access to the true score function. And, as evident from the bounds of \cite{mousavi2023towards}, generating heavy-tailed distributions using Langevin dynamics initialized from the Gaussian distribution is intractable in practice as one needs to run the Langevin dynamics for an exponential number of steps.  We refer to \cite{li2024sharp} for a brief overview of the existing works on the convergence theory of diffusion models.  
   
    \item \cite{seddik2020random} show a form of equivalence between representations generated from Generative Adversarial Networks (GANs) and from GMMs. They considered the Gram matrix of pre-trained classifier representations of the GAN-generated images and show that asymptotically, it possesses the same distribution of eigenvalues as the Gram matrix of Gaussian samples with matching first and second moments.
    \item \cite{loureiro2021learning} investigated the generalization error of linear models for binary classification with logistic loss and $\ell_2$ regularization. On MNIST and Fashion MNIST, they observed a close match between the real images and the corresponding GMM for the linear models and in the feature map of a two-layer network. \cite{loureiro2021learningcurves} considered a student-teacher model and verified universality for the aforementioned datasets via kernel ridge regression. They also explored the output of a deep convolutional GAN (dcGAN), labeling it using a three-layer teacher network. Using logistic regression for classification illustrated a close match with GMMs on GAN-generated data, but a deviation from real CIFAR10 images. \cite{goldt2022gaussian} analyzed the generalization error of Random Features logistic regression using the Gaussian Equivalence property and corroborated their results using images generated by a dcGAN trained on CIFAR100. 
    \cite{pesce2023gaussian} studied the student-teacher model for classification and empirically demonstrated the universality of the double descent phenomenon for MNIST and Fashion MNIST. They preprocessed these datasets using a random feature map, with labels generated by a random teacher, for ridge regression and logistic classification.
    However, they also observed that the universality of the test error fails to hold while using CIFAR10 without preprocessing with either random feature maps, wavelet scattering, or Hadamard orthogonal projection. 
    \cite{dandi2024universality}
    observed that the data distributions generated by conditional GANs trained on Fashion MNIST exhibit Gaussian universality of the test error for generalized linear models. \cite{gerace2024gaussian} considered mixture distributions with random labels and demonstrated universality of test error of the generalized linear models. The universality part of our work can be considered as an exploration of the same phenomenon for conditional diffusion models trained on significantly larger image datasets. 

    \item  \cite{refinetti2023neural} show that SGD learns higher moments of the data as the training continues which exhibited nonuniversality of the test error with respect to the input distribution. Exploring the limitations of current universality results and conditions under which they break remains an interesting direction of research.
    \item \cite{jacot2020kernel} and \cite{bordelon2020spectrum} considered kernel methods for regression and corroborated their findings through experiments on MNIST and Higgs datasets, providing evidence of Gaussian universality.
    
    \item \cite{liunderstanding} explore a connection between diffusion models and GMMs from a different point of view. They observe that if the denoisers are over-parametrized, the diffusion models arrive at the GMMs with the means and covariances matching those of the training dataset, but learn to diverge at later stages of training. Our observations imply that even though the distribution of the diffusion-generated images stops being the same as the corresponding GMM after sufficiently many training steps, they still have properties in common.  
    \item \cite{levi2023underlying} investigated the spectrum of the empirical Gram and covariance matrices of various real world image datasets by centering the images. They demonstrated that the eigendistribution and the level-spacing distribution of the empirical covariance matrix of real data could be closely captured with that of a Wishart matrix whose covariance displays a Toeplitz structure with eigenvalues following a power-law spectrum. Albeit not the main focus of our work, we show that the spectrum of the covariance matrix for each class of data generated by a conventional conditional diffusion model also displays a power-law behavior. Our experiment setting is different in two key aspects. First, by not centering the images, we retain the mean of each class in the data. Second, the empirical covariance is compute per class, rather than over the entire dataset.
    
    \item Concurrent to the submission of this work, \cite{tam2025statistical} established a similar Concentration of Measure Property. While their results are valid for any generative model consisting of Lipschitz operations, they mainly explore concentration properties for GANs. Our paper conducts extensive numerical experiments with diffusion and dives into the question of bounding the Lipschitz constant of the diffusion process after $N$ steps, which is crucial to ensure that the constants in the concentration inequality can be taken to be independent of $N$. Finally, the second part of \cite{tam2025statistical} considers more abstract settings, such as generative models taking general subgaussian noise as input, while in the second part of our work we study Gaussian universality for diffusion-generated data.
\end{itemize}
\section{Main Results}
We start by defining the linear multiclass classification problem and we prove a universality result for such models and state our first empirical observation. We then proceed to provide insights and explanations for why this empirical may hold based on a notion of approximate concentration of measure which requires an analysis of the sampling process of diffusion models (Appendix \ref{subsec: diff}).

\subsection{Classification and Gaussian Universality}\label{subs: univ}
We cover Gaussian universality in the context of linear multiclass classification following the framework described in \cite{ghane2024universality} and extend it to an arbitrary number of classes. As we will see, most known Gaussian universality results operate in an idealized setting that does not appear to be applicable to the covariance matrices estimated from the diffusion-generated images (Figure \ref{fig:cov_spectra}). Nevertheless, we observe empirically that universality holds in the latter setting as well, hence raising a challenge of relaxing the assumptions of the existing universality results to make them more practical. We outline the corresponding notation and challenge below.

\begin{itemize}
    \item Consider data $\bx \in \bbR^d$ being generated according to a mixture distribution with $k$ classes $\bbP = \sum_{i=1} ^k \theta_i \bbP_i$ for $0\le \theta_i \le 1$ and $\sum_{i=1}^k \theta_i =1$. For a sample $\bx$ from $\bbP_i$, i.e the i'th class, we assign a label $\by \in \bbR^k$, to be $\by := \be_i$ (one-hot encoding). 

    \item 
    We consider a linear classifier $\bW \in \bbR^{d\times k}$ with columns $\bw_\ell$ for $\ell \in [k]$ , where for a given datapoint $\bx$, we classify $\bx$ based on
    \begin{align*}
        \argmax_{\ell \in [k]} \bw_\ell^T \bx
    \end{align*}
    The generalization error of a classifier $\bW$ on this task is defined as follows:
    \begin{align*}
        \sum_{i=1}^k \theta_i \bbP\Bigl(i\neq \argmax_{\ell \in [k]} \bw_\ell^T \bx \Bigl| \bx\sim \bbP_i\Bigr)
    \end{align*}
    \item 

    Given a training dataset $\{\bx_i, \by_i\}_{i=1}^n$ with $n$ samples, where each class has $n_i \approx \theta_i n$ samples, we construct the data matrix $\bX \in \bbR^{n\times d}$ and label matrix $\bY \in \bbR^{n\times k}$
    \begin{align*}
        \bX^T := \begin{pmatrix}
            \bx_1 &
            \bx_2 &
            \hdots&
            \bx_n
        \end{pmatrix}, \quad \bY^T :=  \begin{pmatrix}
            \by_1 &
            \by_2 &
            \hdots&
            \by_n
        \end{pmatrix}
    \end{align*}
    Without loss of generality, we can rearrange rows of $\bX$ to group together samples by class. 

    We also consider a Gaussian matrix $\bG \in \bbR^{n\times d}$ whose rows have the same mean and covariances of the corresponding rows in $\bX$. We sometimes refer to this statement as $\bG$ matching $\bX$. In other words, $\bG$ is a matrix of data sampled from the Gaussian mixture model (GMM) defined via $\sum_{i=1}^k \theta_i \calN(\bmu_i, \bSigma_i)$ where $\bmu_i = \bbE_{\bbP_i} \bx$ and $\bSigma_i = \bbE_{\bbP_i} \bx \bx^T - \bmu_i\bmu_i^T$ for $\bx$ belonging to class $i$.

    \item
    To train for $\bW$, we minimize $\|\bY-\bX \bW\|_F^2$ by running SGD with a constant stepsize. By the implicit bias property of SGD \cite{gunasekar2018characterizing}, \cite{azizan2018stochastic} for linear models, we observe that the iterations of SGD initialized from some $\bW_0$ converge to the optimal solution of the following optimization problem
    \begin{align}\label{eq: quad_obj}
        \min_{\bW \in \bbR^{d \times k}} \|\bW-\bW_0\|_F^2 \quad
        s.t \quad \bX \bW = \bY
    \end{align}

Then it is known that under the list of technical Assumptions \ref{ass: uni} listed below the $\bW$ obtained through running SGD on the data matrix $\bX$ has asymptotically the same performance (generalization error) as a $\tilde{\bW}$ obtained through running SGD on the corresponding Gaussian matrix $\bG$, that is $\tilde{\bW}$  solving the following optimization problem: 
 \begin{equation*}
        \min_{\tilde{\bW} \in \bbR^{d \times k}} \|\tilde{\bW}-\bW_0\|_F^2 \quad
        s.t \quad \bG \tilde{\bW} = \bY
    \end{equation*}
    
\end{itemize}
In other words, 
\begin{theorem}\label{thm: univ}
    The following holds  asymptotically in the proportional regime $\frac{d}{n} \rightarrow \delta > 1$ under Assumptions \ref{ass: uni} :
    \begin{align*}
        \Biggl|\sum_{i=1}^k \theta_i \bbP\Bigl(i\neq \argmax_{\ell \in [k]} \bw_\ell^T \bx \Bigl| \bx\sim \bbP_i\Bigr) - \sum_{i=1}^k \theta_i \bbP\Bigl(i\neq \argmax_{\ell \in [k]} \tilde{\bw}_\ell^T \bg \Bigl| \bg\sim \mathcal{N}(\bmu_i, \bSigma_i)\Bigr)\Biggr| \to 0
    \end{align*}
\end{theorem}
\begin{proof}
    See Appendix \ref{sec: App:univ_Thm}.
\end{proof}
The required assumptions are as follows:

\begin{assumption}\label{ass: uni}
Let $\bx$ be any row of $\bX$ and $\bmu$ be its mean. Then: 
\begin{itemize}
 \item  $\|\bmu\|_2 = O(1)$ 
 \item For any deterministic vector $\bv \in \bbR^d$, and $q\in \bbN$, $q \le 6$,  there exists a constant $K > 0$ such that $\bbE_\bx |(\bx-\bmu)^T \bv|^q \le K \frac{\|\bv\|_2^q}{d^{q/2}}$
\item For any deterministic matrix $\bC \in \bbR^{d\times d}$ of bounded operator norm we have $Var(\bx^T \bC \bx) \rightarrow 0$ as $d \rightarrow \infty$ 
\item $s_{\min} (\bX \bX^T) = \Omega(1)$ with high probability where $s_{\min}(\cdot)$ is the smallest singular value.
\end{itemize}

\end{assumption}
\subsection{Limitations of current universality results}\label{subsec: limits}

Assumptions \ref{ass: uni} hold, for example, for any sub-Gaussian $\bx$ with mean and covariance satisfying $\|\bmu\|_2 = O(1)$ and $\frac{c\bI_d}{d}\le \bSigma \le \frac{C\bI_d}{d}$ (see Remark 5 in \cite{ghane2024universality} for details). However, assuming that $\frac{c\bI_d}{d}\le \bSigma \le \frac{C\bI_d} {d}$ is crucial here, as otherwise one can take a Gaussian $\bx$ with $\bSigma = diag(1,\frac{1}{4}, \dots, \frac{1}{d^2})$ and $\bmu = 0$ and notice that $Var(\|\bx\|_2^2) = \tr(\bSigma^2) - \tr(\bSigma)^2$ converges to a strictly positive number, violating the third part of Assumptions \ref{ass: uni} for $\bC = \bI_d$, while $\bx$ is normalized correctly in the sense that $\bbE_{\bx} \|\bx\|^2 = \tr(\bSigma) = O(1)$.

Unfortunately, as can be seen in Figure \ref{fig:cov_spectra}, the spectra of diffusion-generated images look qualitatively similar to the "power law" $\bSigma = diag(1,\frac{1}{4}, \dots, \frac{1}{d^2})$, meaning that Theorem \ref{thm: univ} does not apply in this setting. Moreover, to the best of the authors' knowledge, such covariance matrices break the assumptions commonly made in papers focusing on universality for {\it regression}, which is usually simpler to study. For example, \cite{montanari2022universality} also have to assume $\frac{c\bI_d}{d}\le \bSigma  \le \frac{C\bI_d} {d}$ to get concrete results for over-parametrized regression (cf. Theorem 5 in \cite{montanari2022universality}). Despite this, in the next section, the experimental results seem to suggest the universality of the classification error does not break, This motivates us to relax the Assumptions \ref{ass: uni} in Theorem \ref{thm: univ}. While, technically speaking, universality is proven only for objectives of the form \eqref{eq: quad_obj}, in practice one usually adds a softmax function $S(\bz_1,\dots,\bz_k) = (\dots, \frac{e^{\bz_\ell}}{\sum e^{\bz_i}}, \dots)$
    \begin{equation}\label{eq: softmax_quad_obj}
        \min_{\bW \in \bbR^{d \times k}} \|\bW-\bW_0\|_F^2 \quad
        s.t \quad S(\bX \bW) \approx \bY \nonumber
    \end{equation}

Here, the approximate equality comes from the fact that the coordinates of the range of the softmax cannot turn exactly into zeros but will be very close to it on the training data if one fits the objective \eqref{eq: softmax_quad_obj}. Since this objective is of much greater practical interest than \eqref{eq: quad_obj} and has better convergence properties, we add softmax into the objective for numerical validation of universality in the next section. Note that, from theoretical standpoint, it raises the question of incorporating softmax into the framework of Theorem \ref{thm: univ}.

\begin{emp}\label{emp: univ}
    The test error of the weights trained via minimizing $\|\bY - S(\bX \bW)\|_F^2$ on images generated via EDM diffusion models matches the test error of the weights trained on the matching GMM. The experiments are presented in Figure \ref{fig:classifier} preceded by the description of the setup. 
\end{emp}

\subsection{Approximate Concentration of Measure Property} \label{subs:main_results}

In this section, we provide additional insights leading to Empirical Observation \ref{emp: univ}. Before doing so, we present a definition central to the result of this section. We use the following definition of concentration. Informally, it means that the tails of the distribution decay exponentially fast. Note that it corresponds to the Lipschitz Concentration Property for $q = 2$ from
\cite{seddik2020random}, by setting $c = 0$.

\begin{definition}[{\it ACoM}]\label{def: ACoM}
    Given a probability distribution $x \sim \bbP$ where $\bx \in \mathbb{R}^d$, we say that $\bx$ satisfies the {\it  Approximate Concentration of Measure Property} (ACoM) if there exist absolute constants $\Bigl(C, c, c', d, \sigma\Bigr) > 0$  such that for any $L$-Lipschitz function $f: \bbR^d \to \bbR$ and $s>0$ it holds that 
    \begin{align}\label{eq: acom}
        \bbP\Bigl(|f(\bx) - \bbE f(\bx)| > s \Bigr) \leq Ce^{-(\frac{s}{L\sigma})^2} + ce^{-c'd}
    \end{align}
\end{definition}

The distributions satisfying ACoM arise naturally in many applications and are quite ubiquitous. We appeal to the following proposition proven in \cite{rudelson2013hanson} :

\begin{proposition}\label{prop: gaus_CoM}
    The distribution $\bx \sim \mathcal{N}(0, \bSigma)$ satisfies the ACoM property \ref{def: ACoM}. Moreover, the corresponding $C = 2$, $c=0$, and $\sigma = \|\bSigma^{\frac{1}{2}}\|_{op}$. 
\end{proposition}
Note that by taking $c = 0$, we recover the classical notion of Concentration of Measure in the literature.
If $\bSigma = \frac{\bI_d}{d}$ and $f(\bx) = \|\bx\|_2$, then Proposition \ref{prop: gaus_CoM} implies the classical fact that the norm of the normalized standard vector converges to $1$ in probability as $d \to \infty$ because in this case the upper bound of Definition \ref{eq: acom} becomes $2e^{-(\frac{s}{\sigma})^2} = 2e^{-sd} \to 0$. However, if $\bSigma$ is also normalized as $\tr(\bSigma) = 1$ but $\|\bSigma\|_{op} = \Theta(1)$ (which happens, for instance, if the ordered eigenvalues of $\bSigma$ follow the power law $\lambda_i = \tilde{C}i^{-\alpha}$ for some $\tilde{C}>0$ and $\alpha > 1$), then the variance of $\bx$ does not have to go to zero anymore, but Definition \ref{def: ACoM} still implies that $\bx$ has exponentially decaying tails (to be more precise, $\bx$ is a sub-Gaussian random vector--see Definition 3.4.1 in \cite{vershynin2018high}). Gaussians are far from the only distributions satisfying ACoM; other examples include the strongly log-concave distributions, and the Haar measure-we refer to Section 5.2 in \cite{vershynin2018high} for more examples. The concentration of measure phenomenon has played a key role in the development of many areas such as random functional analysis, compressed sensing, and information theory. 

\textbf{Diffusion:} We denote the denoiser used in the forward and backward processes of a conventional diffusion model as $D_\theta(\cdot)$. In this paper, we focus only on the sampling process where we use $\bx^{(i)}$ to denote the sample generated at the $i$'the step of the sampling process starting from the initialization $\bx^{(0)} \sim \calN(0, t_0^2 \bI_d)$. We run the sampling for $N$ steps. We use $t_{i}$ to denote the value of the noise scale used in the sampling process. Furthermore, $\bx^{(i)}\mapsto \calR^{(i)}_{D_\theta}(\bx^{(i)}, t_{i:i+1}) = \bx^{(i+1)}$ is used to denote the mapping that generates the sample in the next step and we use $ t_{i:i+1}$ to summarize $(t_i, t_{i+1})$. Building on Definition \ref{def: ACoM}, the following result provides insights on why the test error of generalized linear models trained on diffusion-generated images matches that of the matching GMM.


\begin{theorem}\label{thm: ACoM}

    (a) Assume that $\|\nabla_{\bx} D_\theta(\bx^{(i)})\| \le L_D$ holds with probability at least $ 1 - c_1 e^{-c_2d}$ for some $L_D > 0$ and all $i = 0,\dots, N$. Then there exists $L_{\calR} > 0$, such that the following holds as well: 
    \begin{align*} 
        \bbP\Bigl(\|\nabla_{\bx}  \calR^{(i)}_{D_\theta}(\bx^{(i)}, t_{i:i+1})\|_2 \le L_{\calR}\Bigr) \ge 1 - c_1 e^{-c_2 d}
    \end{align*}
    (b)  Furthermore, if $L_{\calR} \le 1$, then the resulting output $\bx^{(N)}$ satisfies the following tail bound for every  $L_f$-Lipschitz $f:\bbR^d \rightarrow \bbR$ and every $s>0$: 
    \begin{align*}
        \bbP\Bigl(|f(\bx^{(N)})- \bbE &f(\bx^{(N)})| > s\Bigr) \le 2 \exp\Big(-\frac{s^2}{2L_f^2}\Bigr) + 2N c_1 e^{-c_2 d}
    \end{align*}
\end{theorem} 

\begin{proof}
    See Appendix \ref{app: thm_CoM}.
\end{proof}

\begin{remark}
    Theorem \ref{thm: ACoM} implies that diffusion generated distributions satisfy the ACoM property from Definition \ref{def: ACoM} as long as $N = O(e^{c_3d})$  for some $c_3 < c_2$. This makes our results applicable to the scenario where the number of steps in the sampling procedure is exponential in the dimension of the data. To motivate the assumption $L_{\calR} \le 1$ made in part (b) of Theorem \ref{thm: ACoM}, we would like to refer the reader to Figure \ref{fig:norm_evolution} as a partial empirical justification for it. We would also like to point out that the norms decrease in such a manner that, in fact, $\|\bx^{(i)} - \tilde{\bx}^{(i)}\|_2 \ll \|\bx^{(0)} - \tilde{\bx}^{(0)}\|_2 \:$ for two independent noise initializations $\bx^{(0)}, \tilde{\bx}^{(0)} \sim \mathcal{N}(0, t_0^2\bI_d)$  and the corresponding points $\bx^{(i)}, \tilde{\bx}^{(i)}$ obtained from $\bx^{(0)}, \tilde{\bx}^{(0)}$ by applying $i$ consecutive steps of the backward process. 
\end{remark}

We have also made the following empirical observation, that partly supports the assumption $L_{\calR} \le 1$ from Theorem \ref{thm: ACoM} as well. Understanding mathematically why Empirical Observation \ref{emp: norms_decrease} holds thus poses an interesting challenge.

\begin{emp}\label{emp: norms_decrease}
    Each sampling step $\bx^{(i+1)} = \calR^{(i)}_{D_\theta}(\bx^{(i)}, t_{i:i+1})$ of the Algorithm \ref{alg:sampling} in Appendix \ref{subsec: diff} decreases norms, i.e. $\|\bx^{(i)}\|_2 \le \|\bx^{(i-1)}\|_2$ throughout the reverse process. The results of the corresponding experiments can be found in Figure \ref{fig:norm_evolution}. 
\end{emp}

We observe this contractivity in the sampling process for the setting described in Appendix \ref{subsec: diff}. This observation raises the possibility that many diffusion models used in practice may also possess a contractive sampling process. An important future direction is to understand the conditions under which this property holds, given the base and data distributions, noise schedule, and denoiser training.

We finally note that the ACoM property in \eqref{eq: acom} is insufficient for concluding universality from any of the known universality theorems even for $c=0$ unless the upper bound from the right-hand side of Definition \ref{def: ACoM} goes to $0$ as $d \rightarrow \infty$. Nevertheless, our experiments suggest that universality holds for diffusion-generated images despite this technicality. As such, we report it as an empirical observation and present the question of extending Theorem \ref{thm: univ} to capture more complicated covariance matrices $\bSigma$ such as the power law as an open question for future theory works. 

\section{Experiments}
Throughout all our experiments, we use the trained conditional diffusion (see Appendix \ref{subsec: diff}) checkpoint from EDM \cite{karras2022elucidating}, which uses the ADM architecture \cite{dhariwal2021diffusion} and was trained on \texttt{Imagenet64} (Imagenet-1k \cite{imagenet} downscaled to $64 \times 64$ pixels).

We take a 20 class subset of the 1000 Imagenet classes and sample 10240 images per class from the diffusion model. Our data is of dimension 12288 $\left( 3 \text{ RGB channels} \times 64 \text{ pixels} \times 64 \text{ pixels} \right)$. We open source the code \footnote{Code available: \url{https://github.com/abao1999/diffusion-gmm}} to reproduce our experiments, and we also log our extensive hyperparameter sweeps for the GLM training \footnote{Sweeps: \url{https://wandb.ai/abao/diffusion-gmm}}.




\subsection{Generalized Linear Models show matching generalization error} \label{subsection:classifier_experiments}
We train generalized linear models (GLMs) on our dataset of diffusion-generated images and on the corresponding Gaussian data sampled from a GMM fitted on 10240 diffusion-generated images per class. Following the setting of subsection \ref{subs: univ}, we use SGD as our optimizer and mean squared error (MSE) as our loss criterion. For multi-class classification, we use a softmax activation on the logits and compute the MSE loss against the one-hot-encoded class labels. For binary classification, we compute the MSE loss on the logit after sigmoid activation. This regression on predicted class probabilities is done in practice as soft or noisy labels and in knowledge distillation \cite{Gou_2021_kd_survey}.

\textbf{Matching Generalization Error: } We observe matching test accuracies for GLMs trained on diffusion-generated images and on the corresponding Gaussian data, over a range of training set sizes and multiple subsets of classes. 

\begin{figure}[htbp]
    \centering
    \includegraphics[width=1.0\linewidth]{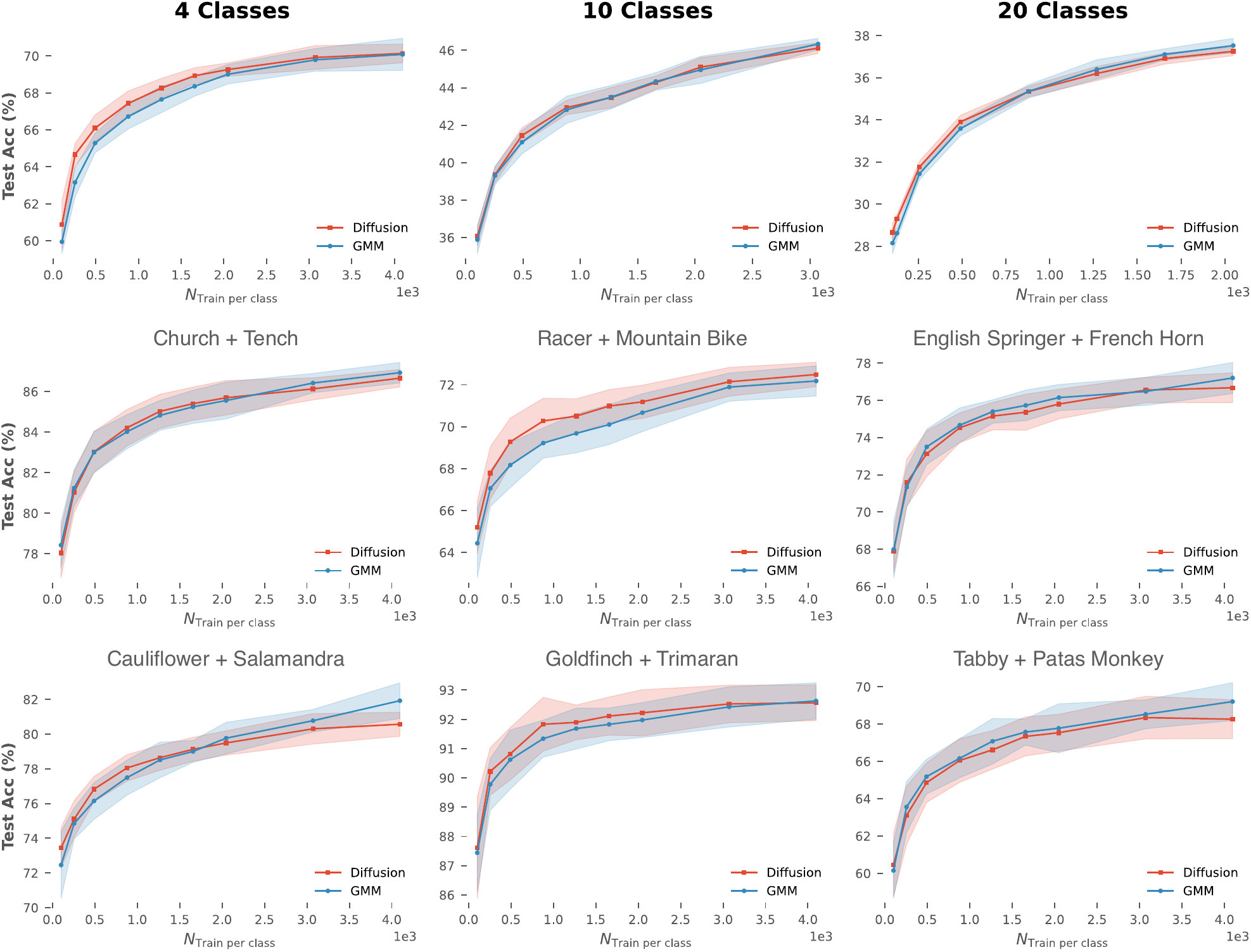}
    \caption{Multiclass (Top Row) and Binary (Middle and Bottom Rows) GLM accuracy for diffusion-generated images {\color{red} (Red)} and GMM samples {\color{blue} (Blue)}. Shaded regions show the standard deviation envelope ($\mu \pm \sigma$) across the 10-20 independent runs per training data split (A total of $\approx 1200$ independent GLM training runs for both diffusion and GMM samples are aggregated in these plots).}
    \label{fig:classifier}
\end{figure}

We compare the accuracies achieved by GLM on the diffusion-generated images versus the GMM samples, when varying the number of samples per class in the training set. Figure \ref{fig:classifier} presents the results of 10-20 independent runs per training data split, with a different random seed for each run. Thus, for each run, a unique pseudorandom generator state determines weight initialization in addition to the sampling and minibatch shuffling of $N_\text{train per class} \in [128, 4096]$ samples from our dataset of 10240 samples per class for diffusion-generated images and GMM samples. Likewise, we fix the size of our held-out test set to $N_{\text{test per class} } = 1024$, randomly sampled according to each run's unique random state from a separate subset of our dataset to ensure no overlap with the training set.

To robustly achieve the best possible classification, we perform an extensive sweep over batch sizes and of learning rates between $[10^{-4}, 0.1]$, with cosine annealing scheduler, while ensuring convergence with respect to the test loss.

The choice of MSE loss on one-hot-encoded labels may seem unconventional for classification but is done to match the setting of \ref{subs: univ}. We repeat our experiments using cross-entropy loss (Fig \ref{fig:classifier_cross_entropy}) and also observe a match, but did not conduct as extensive a sweep and expect the match to improve.

\begin{figure}[H]
    \centering
    \includegraphics[width=1.0\linewidth]{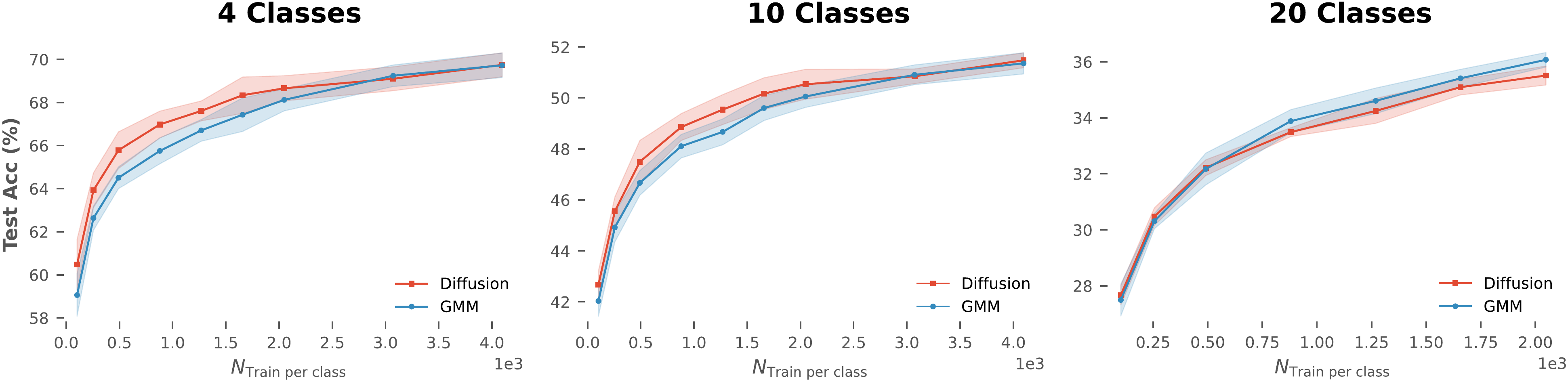}
    \caption{Accuracy for GLM with cross-entropy loss, for diffusion images {\color{red} (Red)} and GMM {\color{blue} (Blue)}.}
    \label{fig:classifier_cross_entropy}
\end{figure}

In Appendix \ref{section:gram spectrum} we compute the eigenvalue spectra of the Gram matrices of multi-class mixtures of diffusion-generated images. Appendix Figures \ref{fig:gram_spectra_plots} and \ref{fig:gram_spectra_eigs_plots} show a very close match between the Gram spectra of diffusion-generated images and that of the corresponding GMM. And in Appendix \ref{section:edm2_gram} we show a close match (Figures \ref{fig:representations_gramians} and \ref{fig:logistic_regression_representations}) between the Gram spectra and the eigenspaces of ResNet representations of high-resolution images sampled from a latent diffusion model, suggesting an investigation into the concentration of these models and their representations as an avenue for future work.

\subsection{Related properties of the sampling process}

\textbf{Norm Evolution through Sampling Process:} We empirically investigate the concentration of norms throughout the sampling process.

\begin{figure}[H]
    \centering
    \includegraphics[width=1.0\linewidth]{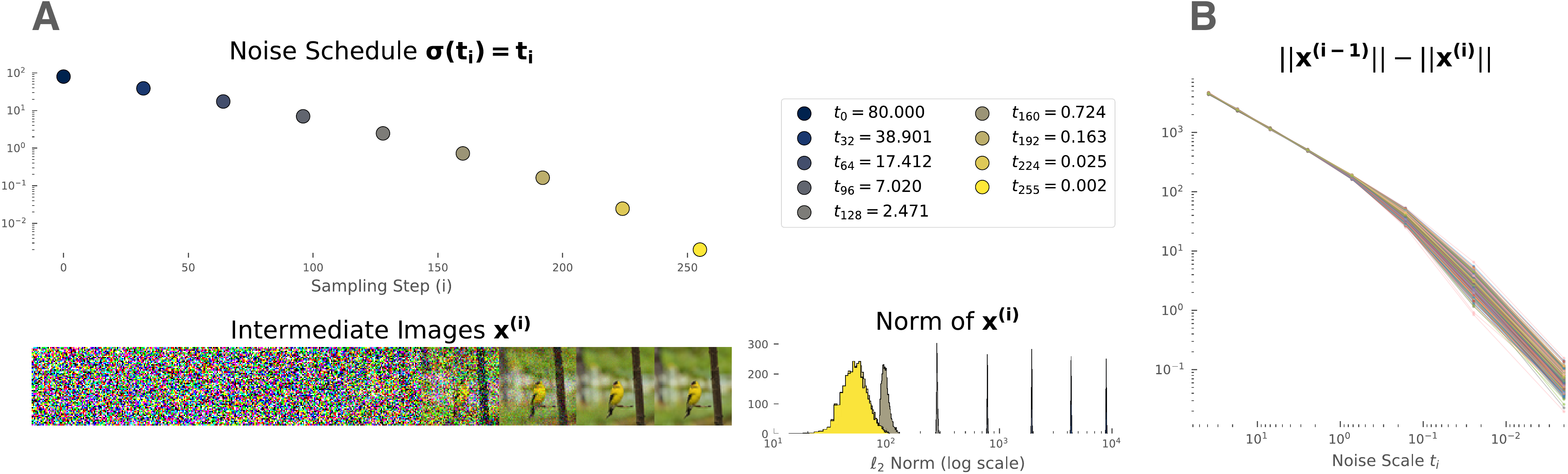}
    \caption{(A) The evolution of the $\ell_2$ norms through the stochastic sampling process. (B) The difference in $\ell_2$ norms of intermediate images between consecutive steps of the EDM sampling process, shown for 10240 generation trajectories (2048 per class, across 5 classes). The sampling process is clearly a contraction, supporting Empirical Observation \ref{emp: norms_decrease}}
    \label{fig:norm_evolution}.
    \vskip -0.15in
\end{figure}

See Appendix \ref{subsec: diff} for an overview of the EDM diffusion sampling process. In Appendix Figure \ref{fig:snapshot_gram_eigs} we show how the sampling process progressively matches the eigenvalues of the Gram matrix. We present further observations regarding the norm evolution and covariance eigenvalues in Appendix \ref{sec:more_obs}. And in Appendix Figure \ref{fig:pixel_norms} we investigate the evolution of the norms of individual pixels.

\textbf{Covariance Spectra of Generated Images exhibit Power Law:} We observe that the top ordered eigenvalues of the covariance matrices for diffusion-generated images follow a power law, supporting our discussion in Section \ref{subsec: limits} on limitations of current universality results.

\begin{figure}[htbp]
    \centering
    \includegraphics[width=0.48\linewidth]{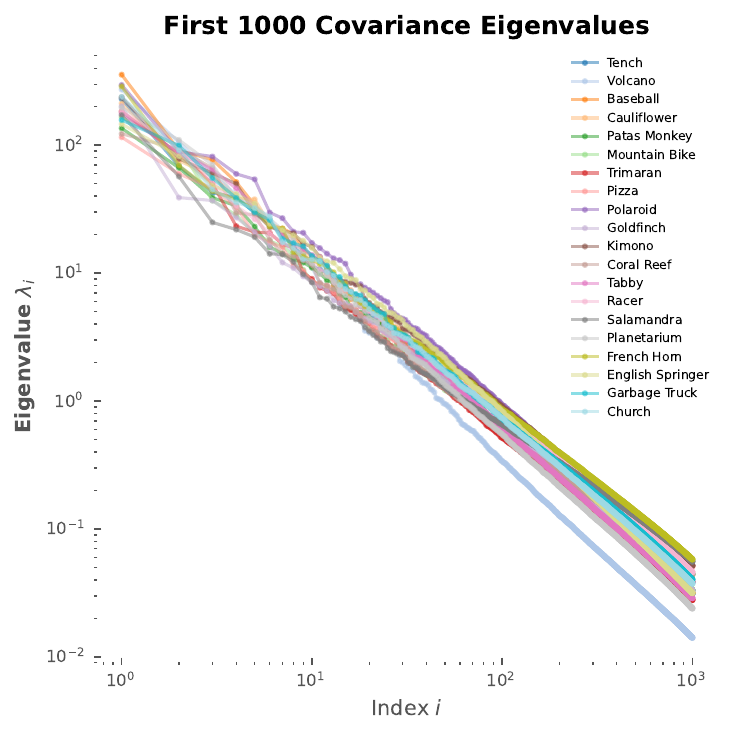}
    \caption{Top 1000 eigenvalues of the covariance matrices computed over 10240 diffusion-generated images per class, shown log-log scale to illuminate power-law behavior, with $i$'th eigenvalue having value $\propto i^ {a}$, $a \in [-\frac{7}{5}, -1]$.}
    \label{fig:cov_spectra}
    \vskip -0.15in
\end{figure}

\section{Conclusion}
    In this work, we focus on the Gaussian universality of the generalization error of generalized linear models on diffusion-generated data. We are motivated by the fact that characterizing the generalization error and performance of neural networks precisely remains one of most challenging problems in modern machine learning. In fact, most theoretical works have focused on analyzing models under specific assumptions about data distribution, such as isotropic Gaussianity even though real-world datasets are almost never Gaussian. As such, we choose to study theoretical properties of the diffusion-generated distributions instead as an approximation to real-world distributions more amenable to analyses. 
    Future directions include extending the universality results to accommodate for more general covariance matrices, incorporating training with softmax into the universality framework and providing a rigorous proof of the contractivity of the sampling process.

\section{Acknowledgments}
We would like to thank Joel A. Tropp for insightful discussions and Morteza Mardani for suggesting to look into the evolution of the norms of the individual pixels during the sampling process. AB was supported by the UT PGEF and the Basdall Gardner Memorial Fellowship. The authors acknowledge the Research Computing Task Force at UT Austin for providing computational resources.

\bibliographystyle{unsrtnat}
\bibliography{refs}


\newpage
\appendix

\section{Diffusion}\label{subsec: diff}
We provide an overview of diffusion models pertinent to our results in this paper. Given samples $x_0 \sim q_0$ from a high-dimensional distribution in $\mathbb{R}^d$, we learn a distribution $p_\theta \approx q_0$ that allows easy sampling. A trained diffusion model essentially applies a sequence of nonlinear mappings (specifically, denoisers, denoted by $D_\theta$) to a white Gaussian input to obtain clean images. Following the formulation in \cite{karras2022elucidating}, assuming the distribution of the training to be "delta dirac", the score function can be expressed in terms of the ideal denoiser that minimizes $L_2$ error for every noise scale, i.e. $\nabla_\bx \log p(\bx; \sigma) = (D_\theta(\bx; \sigma) - \bx) / \sigma^2$. This serves as a heuristic for using $(D_\theta(\bx; \sigma) - \bx) / \sigma^2$ as a surrogate for the score function to run the backward process. In most applications, $D_\theta$ is a neural network trained to be a denoiser, typically using a U-Net backbone. The specific denoiser we consider for our experiments is from ADM \cite{dhariwal2021diffusion} which uses a modified U-Net backbone with self-attention layers. During training, the network sees multiple noise levels, and learns to denoise the images at many scales.  Our analysis and statements in Section \ref{subs:main_results} hold for most of the diffusion models used in practice, as they employ a Lipschitz neural network. \cite{karras2022elucidating} adopt a linear noise schedule $\sigma(t) = t$, but choose a nonlinear step spacing during sampling that emphasizes the low-noise regime. In view of the discussion above, and setting $\sigma(t) = t$, the sampling process is an iterative procedure of $N$ steps:
\begin{equation} \label{eq: sampling}
\begin{split}
    \bx_0 \approx \bx^{(N)} = \calR^{(N-1)}_{D_\theta}\Biggl(\calR^{(N-2)}_{D_\theta}\biggl(\Bigl(\hdots \calR^{(0)}_{D_\theta}\bigl(\bx^{(0)},t_{0:1}\bigr) \hdots\Bigr),t_{N-2:N-1}\biggr),t_{N-1:N}\Biggr)
\end{split}
\end{equation}

Where $\bx_T := \bx^{(0)} \sim \mathcal{N}(0, t_0^2 \bI)$ is isotropic Gaussian noise and $\bx_0 \approx \bx^{(N)}$ is the clean image. We adopt this notation of sub-scripting the time index for $\bx$ while super-scripting its sampler step index in order to avoid confusion with the standard notation in diffusion model papers. At any sampler step $i$ we have a (noisy image, noise level) pair $(\bx^{(i)}, t_i)$ and the next noise level $t_{i+1}$; and $\calR_{D_{\theta}}^{(i)}$ represents the mapping used to generate a less noisy sample i.e. $\bx^{(i+1)} \leftarrow \calR_{D_{\theta}}^{(i)}(\bx^{(i)}, t_{i:i+1})$, which takes in an independent noise $\bepsilon_i$ at each time step, as illustrated by Figure \ref{fig:diffusion-highlevel}.

\begin{figure}[t]
    \centering
    \includegraphics[width=0.75\linewidth]{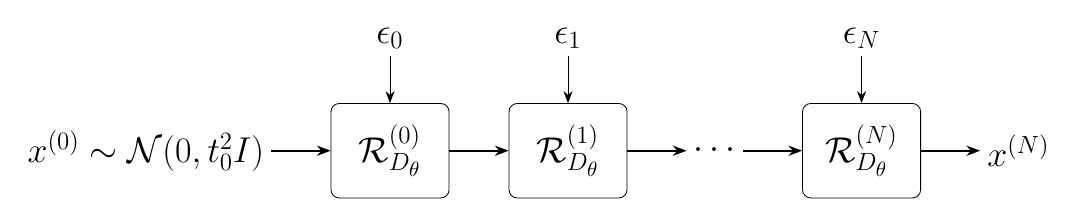}
    \caption{High-level overview of the sampling process}
    \label{fig:diffusion-highlevel}
\end{figure}

\begin{algorithm}[]
\caption{EDM Stochastic Sampler \cite{karras2022elucidating} }\label{alg:sampling}
\Define{$f(\bx, t) := \left(\bx - D_{\theta}(\bx, t)\right) / t $ \tcp*[r]{\scriptsize{Probability Flow ODE}} \label{def:dxdt}}
\Sample $\bx^{(0)} \sim \calN(0, t_0^2 \bI)$ \\
\For{ $i \in \{0, \ldots , N-1 \} $ }{
\Sample $\bepsilon \sim \calN(0,S_{\text{noise}}^2 \bI)$\\ 
$\hbx^{(i)} \leftarrow \bx^{(i)} + t_i \sqrt{\gamma(2+\gamma)} \bepsilon$  \tcp*[r]{\scriptsize{Inject Noise}} 
$h_i \leftarrow t_{i+1} - t_i (1+\gamma)$ \tcp*[r]{\scriptsize{Step Size}} 
$\bx^{(i+1)} \leftarrow \hbx^{(i)} + h_i f(\hbx^{(i)}, t_i(1+\gamma))$ \tcp*[r]{\scriptsize{Euler Step}}
\If{$t_{i+1} \neq 0$}{
    $\bx^{(i+1)} \leftarrow \hbx^{(i)} + \frac{h_i}{2} \left( f(\hbx^{(i)}, t_i(1+\gamma)) + f(\bx^{(i+1)}, t_{i+1}) \right)$ \tcp*[r]{\scriptsize{Second-order correction}}
}
}
\Return{$\bx^{(N)}$} 
\end{algorithm}

We focus on this framework and observe that in summary, $\bx^{(i+1)} \leftarrow \calR_{D_{\theta}}^{(i)}(\bx^{(i)}, t_{i:i+1})$, with
\begin{align*}
    \calR_{D_{\theta}}^{(i)} (\bx^{(i)}, t_{i:i+1}) := \hbx^{(i)} + \frac{h_i}{2t_{i+1}} \Bigl[ \hbx^{(i)} + (h_i + t_{i+1}) d_i - \underbrace{D_{\theta}(\hbx^{(i)} + h_i d_i, t_{i+1})}_{\text{Denoiser after Euler step}} \Bigr]
\end{align*}
Where $d_i := f(\hbx^{(i)}, t_{i}(1 + \gamma))$ is as defined in line \ref{def:dxdt} of Algorithm \ref{alg:sampling} and $\gamma$ is a hyperparameter controlling the amount of additional injected noise whose scale is determined by the $S_{\text{noise}}$ hyperparameter. And $\hbx^{(i)}$ is the current image with the added noise. Formally, we would like to claim  that the distribution of the output $\bx^{(N)}$ satisfies ACoM, and we visualize the evolution of the norms of these quantities through the sampling process in Figure \ref{fig:norm_evolution} to further illuminate our argument about the $1$-Lipschitznes of the generative process.

Following the recommendations of \cite{karras2022elucidating}, the stochastic sampling process of $N = 256$ steps begins with $t_{max} := t_{0} = 80$ and ends with $t_{min} := t_{N-1} = 0.002$. The sampling step schedule is constructed as $t_{i < N} = \left( {t_{max}}^{\frac{1}{\rho}} + \frac{i}{N-1} \left( {t_{min}}^{\frac{1}{\rho}} - {t_{max}}^{\frac{1}{\rho}} \right) \right)^{\rho}$. Here, $\rho=7$ is a hyperparameter observed to improve image quality.

\section{Proof of Theorem \ref{thm: univ}}\label{sec: App:univ_Thm}

In order to study the problem presented in \eqref{eq: quad_obj}, we use a Lagrange multiplier variable $\lambda \in \bbR$ to bring in the constraint:
\begin{align*}
    \Phi(\bA) := \min_{\bW \in  \bbR^{d \times k}}&\|\bW-\bW_0\|_F^2  \\
        s.t \quad &\bX \bW = \bY \\
        = \min_{\bW \in \bbR^{d \times k}} &\sup_{\lambda>0}  \frac{\lambda}{2} \|\bA \bW - \bY\|_F^2 +  \|\bW-\bW_0\|_F^2
\end{align*}
Now we will consider the following ridge regression objective:
\begin{align*}
    \Phi_\lambda(\bA) &:= \min_\bW \frac{\lambda}{2}\|\bA \bW - \bY\|_F^2 + \|\bW\|_F^2  \\ &= \sum_{\ell=1}^k \min_{\bw_\ell} \frac{\lambda}{2}\|\bA \bw_\ell - \by_\ell\|_2^2 + \|\bw_\ell\|_2^2 
\end{align*}
Note that $\Phi(\bA) = \sup_{\lambda > 0} \Phi_\lambda(\bA)$, therefore, we analyze $\Phi_\lambda(\bA)$ for every $\lambda > 0$ and via a uniform convergence argument, we extend the result to $ \sup_{\lambda > 0} \Phi_\lambda(\bA)$. We denote the solution to the above optimization problem as $\bW_{\Phi_\lambda(\bA)}$. For the main quantity of interest, the generalization error,
\begin{align}\label{id: gen_error}
    &\biggl|\bbP\Bigl(i\neq \argmax_{\ell \in [k]} \bw_{\ell, \Phi_\lambda(\bX)}^T \bx \Bigl| \bx\sim \bbP_i\Bigr) \nonumber \\ &- \bbP\Bigl(i\neq \argmax_{\ell \in [k]} \bw_{\ell, \Phi_\lambda(\bG)}^T \bg \Bigl| \bg\sim \mathcal{N}(\bmu_i, \bSigma_i)\Bigr) \biggr|,
\end{align} 
We will leverage a result from literature. Namely, we utilize a multi-dimensional version of the CLT result of \cite{bobkov2003concentration} (Corollary 2.5) which controls the following quantity for a matrix $\bW$ with "generic" column vectors: 
\begin{align}\label{id: gen_clt}
    \biggl|&\bbP\Bigl(i\neq \argmax_{\ell \in [k]} \bw_{\ell, \Phi_\lambda(\bX)}^T \bx \Bigl| \bx\sim \bbP_i\Bigr) \nonumber \\&- \bbP\Bigl(i\neq \argmax_{\ell \in [k]} \bw_{\ell, \Phi_\lambda(\bX)}^T \bg \Bigl| \bg\sim \mathcal{N}(\bmu_i, \bSigma_i)\Bigr) \biggr|,
\end{align}
This extension follows by applying a union bound argument. Hence, using \eqref{id: gen_clt}, to analyze \eqref{id: gen_error}, we would only need to bound the following:
\begin{align*}
    \biggl|&\bbP\Bigl(i\neq \argmax_{\ell \in [k]} \bw_{\ell, \Phi_\lambda(\bX)}^T \bg \Bigl| \bg\sim \mathcal{N}(\bmu_i, \bSigma_i)\Bigr)  \\&- \bbP\Bigl(i\neq \argmax_{\ell \in [k]} \bw_{\ell, \Phi_\lambda(\bG)}^T \bg \Bigl|\bg\sim \mathcal{N}(\bmu_i, \bSigma_i)\Bigr) \biggr|
\end{align*}
Which involves analyzing the covariance and the mean of $\bw_{\ell, \Phi_\lambda(\bA)}^T \bg$ for $\bA = \bG, \bX$. Note that in the argument above $\bA$ could be either $\bX$ and $\bG$ as a multi-dimensional CLT argument reduces the problem of universality of the test error on $\bX$ to $\bG$ and it only requires the first and second order statistics of $\bX$.

We know from \cite{thrampoulidis2020theoretical} for the case of a GMM, (see Equation 2.7 in \cite{thrampoulidis2020theoretical}), which corresponds to taking $\bA = \bG$, the generalization error is characterized by the quantities $\bmu_\ell^T (\bw_\ell - \bw_{\ell'})$, and $\bSigma^{1/2}_\ell \bS\bSigma^{1/2}_\ell$ where $(S_\ell)_{ij} := (\bw_i-\bw_j)^T (\bw_i - \bw_j)$ for $i,j \neq \ell$. Now in order to characterize $(S_\ell)_{ij}$, note that we need to understand the pairwise interaction of $\bw_i$ and $\bw_j$ and the decomposition provided cannot capture these quantities. To do this, we use the following identity:
\begin{align*}
    \min_{\bw_i, \bw_j} \frac{\lambda}{2}\|\bA \bw_i - & \by_\ell \|_2^2  + \|\bw_i\|_2^2 + \frac{\lambda}{2}\|\bA \bw_j - \by_\ell\|_2^2 + \|\bw_j\|_2^2 \\ = \min_{\bw_i-\bw_j, \bw_i+\bw_j} &\frac{\lambda}{4}\|\bA (\bw_i+\bw_j) - \by_\ell\|_2^2 + \|\bw_i+\bw_j\|_2^2 \\ &+ \frac{\lambda}{4}\|\bA (\bw_i-\bw_j) - \by_\ell\|_2^2 + \|\bw_i-\bw_j\|_2^2 
\end{align*}
And by studying the norms of $\bSigma_\ell^{1/2}(\bw_i\pm\bw_j)$ we can recover $(S_\ell)_{ij}$.  Therefore, we need to prove a universality result for the optimization in the right hand side of above. But this follows by combining the results of \cite{ghane2024universality} and the above identity, we observe that $|(S_\ell(\bX))_{ij}-(S_\ell(\bG))_{ij}|\rarrowp 0$ for every $\lambda$.
Using the uniform convergence result from Section B.2 in the appendix of \cite{ghane2024universality}, we observe that if $\Phi_\lambda(\bA) \rarrowp c_\lambda$ then $\sup_{\lambda >0} \Phi_\lambda(\bA) \rarrowp \sup_{\lambda>0} c_\lambda$. Now, to conclude the proof, we combine this fact with a perturbation argument.

\section{Proof of Theorem \ref{thm: ACoM}} \label{app: thm_CoM}

To prove Theorem \ref{thm: ACoM}, we will use the following auxiliary result.
\begin{lemma}\label{lemma: norm_decrease}
    \begin{enumerate}[label=(\alph*) ]
        \item Assume that sampling process is conducted in such a way that for some $c_1, c_2 > 0$, for every $1\le i \le N$, 
        \begin{align*}
            \bbP\Bigl(\| \calR^{(i)}_{D_\theta}(\bx^{(i)}, t_{i:i+1})\|_2 \le \|\bx^{(i)}\|_2\Bigr) \ge 1 - c_1 e^{-c_2 d}.
        \end{align*}
        Then the probability of the whole sampling process being norm-decreasing is at least $1-N c_1 e^{-c_2 d}$.
        \item If $x$ satisfies ACoM with $\Bigl(C, c, c', d, \sigma\Bigr) > 0$ and $f(\cdot)$ is diffenentiable and $L_f$-Lipschitz with probability at least $1-\tilde{c}e^{-\hat{c} d}$, then $f(x)$ satisfies the ACoM with $\Bigl(C, c+\tilde{c}, \min\{c', \hat{c}\}, d, L_f \sigma\Bigr)$.
    \end{enumerate}
\end{lemma}

\begin{proof}[Proof of Lemma \ref{lemma: norm_decrease}]
The proof of part (a) follows by a straightforward application of the union bound. In fact,
\begin{align*}
    \bbP\biggl( \bigcap_{i=1}^N \Bigl\{ \| &\calR^{(i)}_{D_\theta}(\bx^{(i)}, t_{i:i+1})\|_2 \le \|\bx^{(i)}\|_2\Bigr\} \biggr) \\ &= 1 - \bbP\biggl( \bigcup_{i=1}^N \Bigl\{ \| \calR^{(i)}_{D_\theta}(\bx^{(i)}, t_{i:i+1})\|_2 \le \|\bx^{(i)}\|_2\Bigr\} \biggr)
\end{align*}
Now since
\begin{align*}
    \bbP\biggl( \bigcup_{i=1}^N \Bigl\{ \| \calR^{(i)}_{D_\theta}&(\bx^{(i)}, t_{i:i+1})\|_2\le \|\bx^{(i)}\|_2\Bigr\} \biggr) \\ &\le \sum_{i=1}^N \bbP\biggl( \Bigl\{ \| \calR^{(i)}_{D_\theta}(\bx^{(i)}, t_{i:i+1})\|_2 \le \|\bx^{(i)}\|_2\Bigr\} \biggr) \\ &\le N c_1 e^{-c_2 d}
\end{align*}
Part (a) follows.

For part (b), we use the following identity for arbitrary $L_g$-Lipschitz function $g(\cdot)$ 
\begin{align*}
    &\bbP\Bigl(|(g\circ f)(x)-\bbE (g\circ f) (x)| > t\Bigr) \\ & \le \bbP\biggl(\Bigl\{ |(g\circ f)(x)-\bbE (g\circ f)(x)| > t \Bigr\} \cap \Bigl\{ \|\nabla f(x)\|_2 \le L_f \Bigr\} \biggr)  \\ &+ \bbP\Bigl(\Bigl\{ \|\nabla f(x)\|_2 > L_f\Bigr\} \Bigr)
\end{align*}
We have for the first term:
\begin{align*}
    &\bbP\biggl(\Bigl\{ |(g\circ f)(x)-\bbE (g\circ f)(x)| > t \Bigr\}  \cap \Bigl\{\|\nabla f(x)\|_2 \le L_f\Bigr\} \biggr)  \\ &= \bbP\biggl(\Bigl\{ |(g\circ f)(x)-\bbE (g\circ f)(x)| > t \Bigr\} \Bigl| \Bigl\{\|\nabla f(x)\|_2 \le L_f\Bigr\} \biggr) \cdot \bbP\Bigl(\Bigl\{\|\nabla f(x)\|_2 \le L_f\Bigr\} \Bigr) \\ & \le  \bbP\biggl(\Bigl\{ |(g\circ f)(x)-\bbE (g\circ f)(x)| > t \Bigr\} \Bigl| \Bigl\{\|\nabla f(x)\|_2 \le L_f\Bigr\} \biggr) \\  &\le C \exp\Bigl({-\frac{t^2}{\sigma^2 L_g^2L_f^2}}\Bigr)+ c e^{-c' d}
\end{align*}
For the second term from the assumption,
\begin{align*}
    \bbP\Bigl(\Bigl\{\|\nabla f(x)\|_2 > L_f\Bigr\} \Bigr) \le \tilde{c}e^{-\hat{c} d}
\end{align*}
Summarizing
\begin{align*}
    \bbP\Bigl(|(g\circ & f)(x)-\bbE (g\circ f) (x)| > t\Bigr) \\ &\le C \exp\Bigl({-\frac{t^2}{\sigma^2 L_g^2L_f^2}}\Bigr)+ c e^{-c' d} + \tilde{c}e^{-\hat{c} d} 
    \\ &\le  C \exp\Bigl({-\frac{t^2}{\sigma^2 L_g^2L_f^2}}\Bigr) + (c+\tilde{c}) \exp\Bigl(-\min\{c', \hat{c}\}d\Bigr)
\end{align*}
\end{proof}

\begin{proof}[Proof of Theorem \ref{thm: ACoM}]
    (a) follows by noting that $\calR^{(i)}_{D_\theta}(\bx^{(i)}, t_{i:i+1})$ is a deterministic Lipschitz function of $D_\theta(\cdot)$.

    (b) 
    The proof of the part (b) of Theorem \ref{thm: ACoM} follows by combining the parts (a) and (b) of Lemma \ref{lemma: norm_decrease} and noting that the sampling process will be $1$-Lipschitz and iteratively applying the result of part (b) of Lemma \ref{lemma: norm_decrease}.
\end{proof}

\section{Application of Theorem \ref{thm: ACoM} to conventional samplers}

Recall that the images are generated step by step according to:

\begin{equation}
    \bx_0 \approx \bx^{(N)} = \calR^{(N-1)}_{D_\theta}\Biggl(\calR^{(N-2)}_{D_\theta}\biggl(\Bigl(\hdots \calR^{(0)}_{D_\theta}\bigl(\bx^{(0)},t_{0:1}\bigr) \hdots\Bigr),t_{N-2:N-1}\biggr),t_{N-1:N}\Biggr) \nonumber
\end{equation}
We will prove that each $\bx^{(i)}$ satisfies the ACoM property by induction by $i = 0, \dots, N$:

$\bullet$ Basis $i = 0$ follows from Proposition \ref{prop: gaus_CoM}

$\bullet$ Step $i \to i + 1$ follows by applying Lemma \ref{lemma: norm_decrease} with $L = 1$ to  $\bx^{(i+1)} = \calR^{(i)}_{D_\theta}(\bx^{(i)}, t_{i:i+1})$ if we verify that $\calR^{(i)}_{D_\theta}(.,t_{i:i+1})$ is $1$-Lipschitz.

To show that each step $\calR^{(i)}_{D_\theta}(\cdot)$ is  Lipschitz with high probability, recall that it is comprised of the noise injection followed by the application of the denoiser. The effect of the noise injection part is covered by Lemma \ref{lemma: ACoM_coup} below. As for the trained denoiser $D_\theta(\cdot)$ used in $\calR_{D_\theta}(\cdot)$, it need not be Lipschitz in general, but in the case of the EDM sampler as well as many other samplers, one could observe that the trained network is Lipschitz with \emph{high probability} {\it assuming that each part of the architecture is Lipschitz with high probability}. Thus, we illustrate the argument for EDM samplers for which each part of the architecture of the denoiser is Lipschitz with high probability.

We also test the validity of our results empirically for the U-Net architecture because the latter is often employed in practice. Despite our prediction matches the empirical results well for this architecture too, note that we were not able to verify the high-probability Lipschitness assumptions for this network, leaving it as an interesting challenge that will hopefully stimulate future work in this direction. To be precise, U-Net is a neural network consisting of the following blocks:

\begin{itemize}
    \item {\it Fully-Connected Layers} with a Lipschitz activation function $\sigma = SiLU$ and a matrix of weights $\bW$. These are Lipschitz functions with constant $\|\sigma\|_{Lip}\|\bW\|_{op}$.
    \item {\it Convolutional Layers} with a filter $\bW$. These are also Lipschitz functions with constant $\|\sigma\|_{Lip}\|\bW\|_{op}$. 
    \item {\it Self-Attention Layers} As shown in \cite{kim2021lipschitz}, these are {\it not Lipschitz} over the entire domain. However, intuitively it should be true that, if we restrict the domain to points sampled from a distribution satisfying ACoM, then it {\it is Lipschitz with high probability}. We have not been able to verify it by the moment of submission and encourage an interested reader to do so. To support our point further, if this architecture were not {\it Lipschitz for the inputs of interest with high probability}, one would imagine that it would be utterly impractical to train and use. However, nobody has reported this for self-attention. 
    \item {\it Max Pool, Average Pool, Group Normalization, Positional Embedding, Upsampling and Downsampling Layers}. All these layers are $1$-Lipschitz.

\bigskip

\end{itemize}

We conclude that the mapping $\bx^{(i)} \to \bx^{(i+1)} $ is Lipschitz but, technically speaking, the constant is unbounded. Moreover, since the sampling process involves $N$ steps for a relatively large $N$, the Lipschitz constant of the mapping $\bx_T \to \bx_0$ might accumulate and explode unless the Lipschitz constant of each step is bounded by $1$. While we could not prove directly that the latter is the case so far, we observed $\calR^{(i)}_{D_\theta}$ to be contractive in the simulations we have conducted (cf. Figure \ref{fig:norm_evolution}). As such, we decided to assume that the training is performed in such a manner that the sampling steps $\bx^{(i)} \to \bx^{(i+1)} $ are all $1$-Lipschitz mappings for the scope of the present work. 

\begin{lemma}\label{lemma: ACoM_coup}
    If $(x,y) \sim \Pi$ where $\Pi$ is a product measure with marginals, $\pi_{1\#} \Pi = p_1$ and $\pi_{2\#} \Pi = p_2$ , $p_1$ and $p_2$ are two distributions satisfying ACoM with $\Bigl(C_1, c_1, c'_1, d, \sigma_1\Bigr)$ and $\Bigl(C_2, c_2, c'_2, d, \sigma_2\Bigr)$, respectively, then $(x,y)$ also satisfies ACoM with $\Bigl(C_1+C_2, c_1+c_2, \min\{c'_1, c'_2\}, d, \max\{\sigma_1,\sigma_2\}\Bigr)$.
\end{lemma}
\begin{proof}
    The proof technique is adopted from \cite{ledoux2001concentration}. 
    For every $L$-Lipschitz function $f: \bbR^d \times \bbR^d \rightarrow \bbR$, we have from triangle inequality
    \begin{align} \label{ineq: CoM_decomp}
        \bbP\Bigl( |f(x,y) - \bbE_{\Pi} &f(x,y)|>2t\Bigr)  \nonumber \\ &\le \bbP\Bigl( |f(x,y) - \bbE_{p_1} f(x,y)| >t\Bigr) \nonumber \\ &+ \bbP\Bigl( |\bbE_{p_1} f(x,y) - \bbE_{\Pi} f(x,y)| >t\Bigr)
    \end{align}
    For the first term in \ref{ineq: CoM_decomp}, we have that for the product measure $\Pi$
    \begin{align*}
        \bbP\Bigl( |f(x,y) &- \bbE_{p_1} f(x,y)|  >t\Bigr) \\&= \bbE_{\Pi} \mathds{1}\Bigl\{ |f(x,y) - \bbE_{p_1} f(x,y)| >t \Bigr\}\\& = \bbE_{p_2} \bbE_{p_1|p_2} \mathds{1}\Bigl\{ |f(x,y) - \bbE_{p_1} f(x,y)| >t \Bigr\} \\ &= \bbE_{p_2} \bbP_{p_1|p_2}\Bigl(|f(x,y) - \bbE_{p_1} f(x,y)| >t\Bigr)
    \end{align*}
    Now we observe that for every $y$, $f(x,y)$ is also $L$-Lipschitz in $x$, thus by ACoM
    \begin{align*}
        \bbP\Bigl( |f(x,y) - \bbE_{p_1} &f(x,y)| >t\Bigr) \\ &= \bbE_{p_2} \bbP_{p_1|p_2}\Bigl(|f(x,y) - \bbE_{p_1} f(x,y)| >t\Bigr) \\& \le C_1 e^{-(\frac{t}{L\sigma_1})^2} + c_1 e^{-c_1' d} 
    \end{align*}
    For the second term in \ref{ineq: CoM_decomp}, letting $g(y) := \bbE_{p_1} f(x,y)$, we observe that $g$ is also Lipschitz, so by ACoM for $p_2$, 
    \begin{align*}
        \bbP\Bigl( |\bbE_{p_1} f(x,y) - \bbE_{\Pi} f(x,y)| >t\Bigr) \le  C_2e^{-(\frac{t}{L\sigma_2})^2} +  c_2 e^{-c_2' d} 
    \end{align*}
    Summarizing, we obtain that:
    \begin{align*}
         \bbP\Bigl( |f(x,y) &- \bbE_{\Pi} f(x,y)|>2t\Bigr) \\ &\le C_1 e^{-(\frac{t}{L\sigma_1})^2} + c_1 e^{-c_1' d} + C_2e^{-(\frac{t}{L\sigma_2})^2} +  c_2 e^{-c_2' d} \\ &\le (C_1 + C_2) \exp\Bigl(-\Bigl(\frac{t}{L\max\{\sigma_1, \sigma_2\}}\Bigr)^2\Bigr) \\ &+ (c_1 + c_2) \exp\Bigl(-\min\{c_1', c_2'\} d\Bigr)
    \end{align*}
    Which concludes the proof.
\end{proof}

\section{Gram Spectrum} \label{section:gram spectrum}
For each of the subsets of classes we considered for our multiclass experiments, we also investigated the spectrum of the Gram matrix of the corresponding mixture distribution. Using an equal number of samples per class, we construct a data matrix $\bX \in \mathbb{R}^{n \times d}$ where $n$ is the total number of samples and $d=12288$ is the dimensionality of each sample (viewed as a vector). Figure \ref{fig:gram_spectra_plots} presents the eigenvalue spectrum of the resulting Gram matrix of the type $\bX\bX^T \in \mathbb{R}^{n \times n}$. As can be seen, we observe a very close match between the distributions of the eigenvalues of the Gram matrices for the diffusion-generated data and the corresponding GMM, but there is a slight mismatch for the smaller eigenvalues. We leave the question of finding out if there are any reasons for the latter mismatch apart from numerical inaccuracies for future work. 
\begin{figure*}[htbp]
    \centering
    \includegraphics[width=1.0\linewidth]{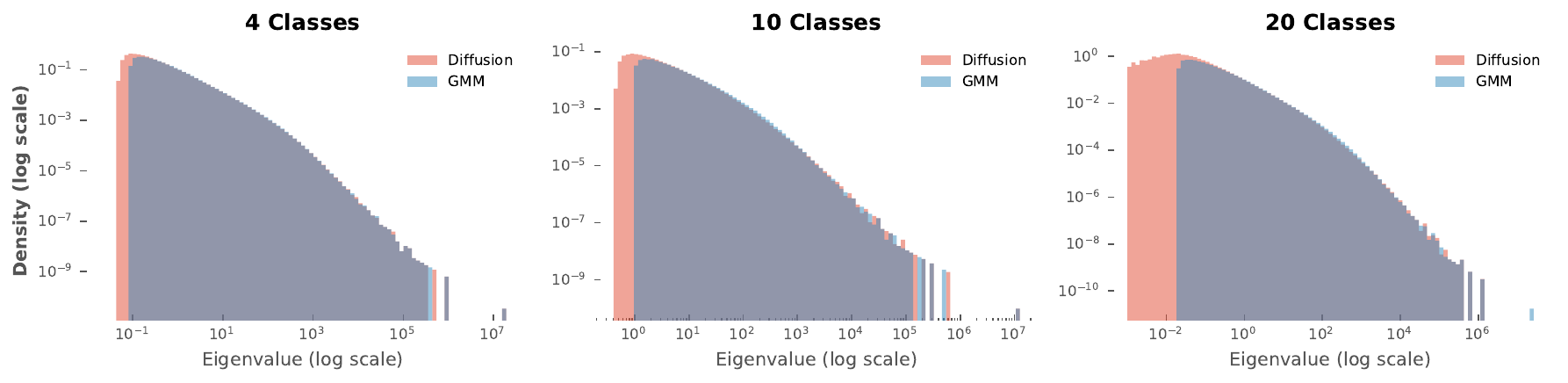}
    \caption{Spectra of Gram Matrices for balanced mixtures of 4, 10, and 20 classes, for Diffusion {\color{red} (Red)} and GMM {\color{blue} (Blue)}. We use 2048 samples per class for the 4-class mixture, and 512 samples per class for 10 and 20 class mixtures.}
    \label{fig:gram_spectra_plots}
\end{figure*}
\begin{figure*}[htbp]
    \centering
    \includegraphics[width=1.0\linewidth]{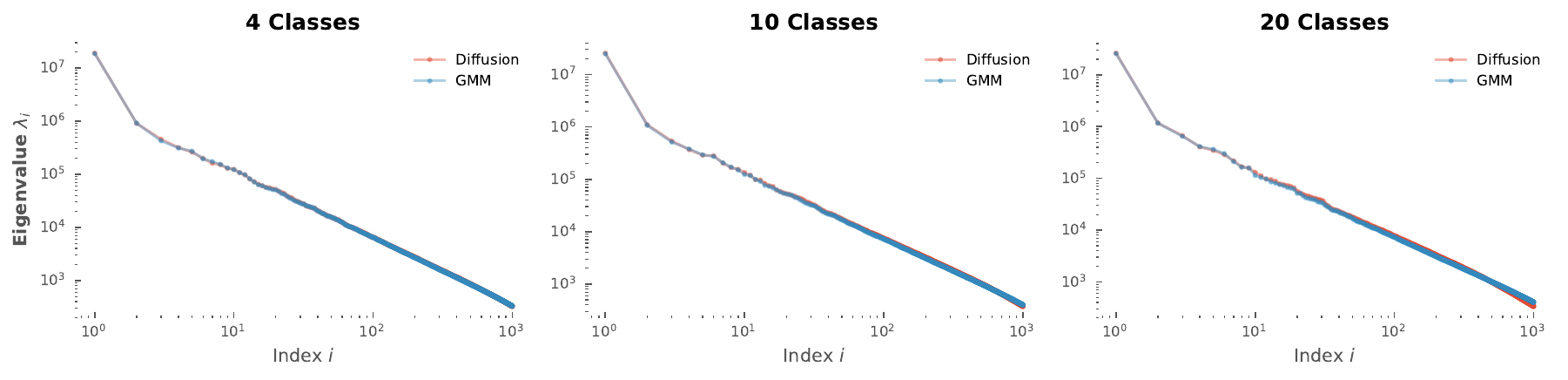}
    \caption{First 1000 eigenvalues of Gram Matrices for balanced mixtures of 4, 10, and 20 classes.}
    \label{fig:gram_spectra_eigs_plots}
\end{figure*}

Note that while establishing the closeness of eigenvalue distributions of the Gram matrices allows one to characterize the behavior of certain algorithms such as Least-Squares SVM or spectral clustering, this does not allow us to analyze more elaborate algorithms. For example, the LASSO objective $\min_{\bw} \|\bX\bw-\by\|_2^2 + \lambda \|\bw\|_1$ for $\bw \in \bbR^d, \bX \in \bbR^{n \times d}$ is not unitarily invariant. Hence, given $\bX' \in \bbR^{n \times d}$, even knowing that the Gram matrices $(\bX')^T\bX'$ and $\bX^T\bX$ are exactly equal to each other does not let one conclude that $\bw'$ identified via 
$\min_{\bw'} \|\bX'\bw'-\by\|^2_2 + \lambda \|\bw'\|_1$ yields performance similar to the performance of $w$.

\textbf{Evolution of Gram Matrix Eigenvalues through Sampling:}
We observe that the sampling process first matches the top Gramian eigenvalues, and progressively matches the lower eigenvalues.
\begin{figure}[H]
    \centering
    \includegraphics[width=1.0\linewidth]{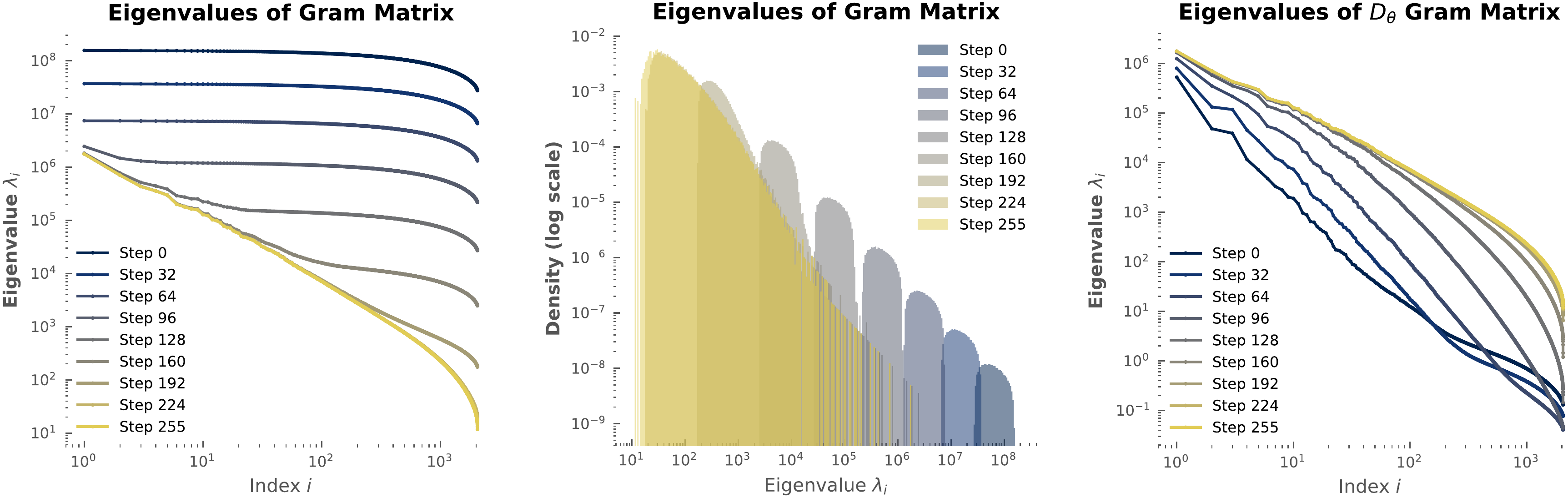}
    \caption{Gramian eigenvalues through EDM sampling. Computed with 2048 samples of one class}
    \label{fig:snapshot_gram_eigs}%
\end{figure}

\section{More Observations} \label{sec:more_obs}
\begin{figure}[H]
    \centering
    \includegraphics[width=1.0\linewidth]{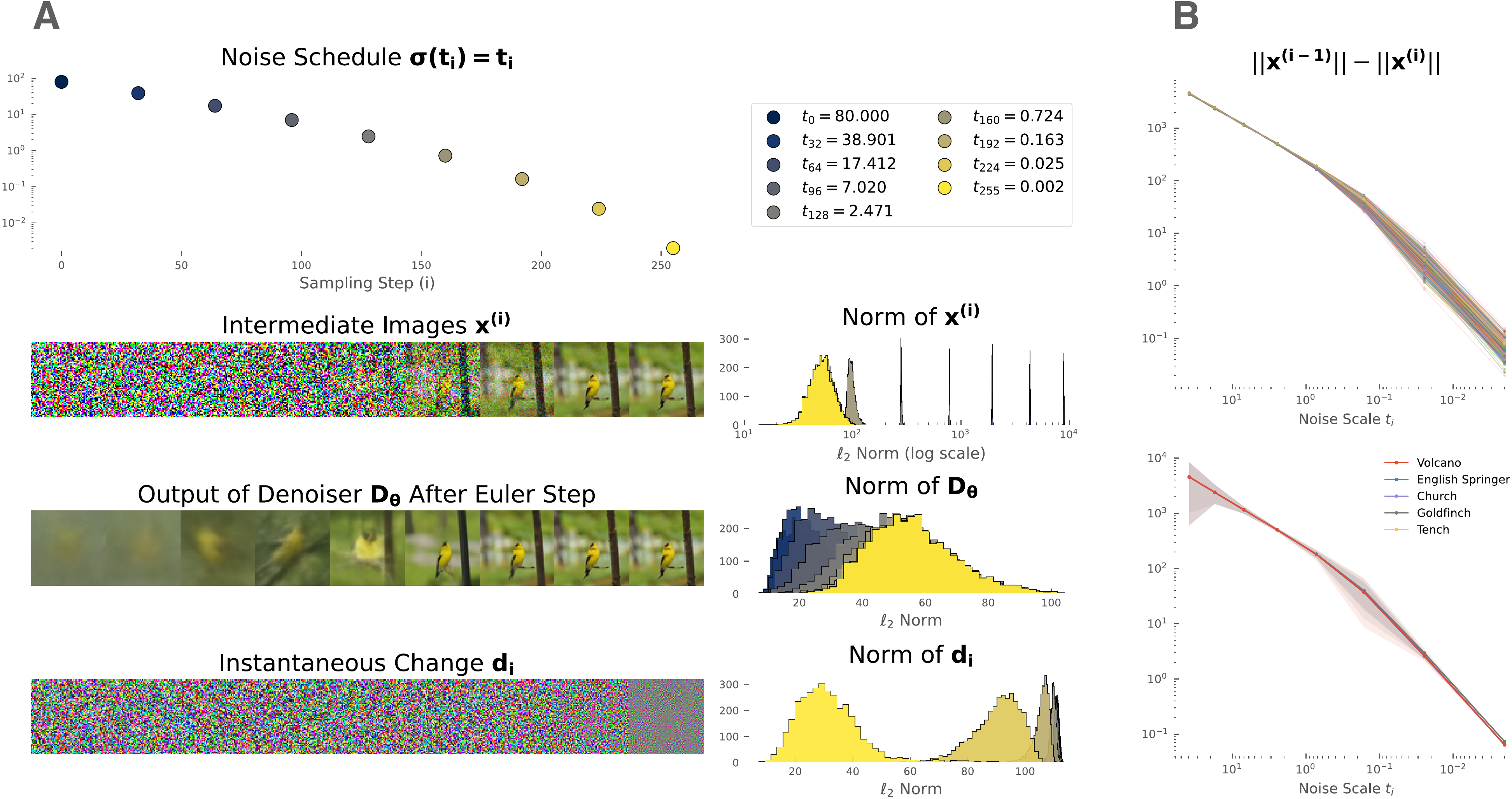}
    \caption{(A) The evolution of the $\ell_2$ norms through the stochastic sampling process (See algorithm \ref{alg:sampling} for definition of $D_{\theta}$ and $d_i$). (B) Top: The difference in $\ell_2$ norms of intermediate images between consecutive steps of the EDM sampling process. (B) Bottom: The mean norm decrease per class, using 2048 samples per class; shaded regions represent the variance of the norm decrease per class.}
    \label{fig:norm_evolution_more} 
\end{figure}

\begin{figure}[htbp]
    \centering
    \includegraphics[width=1.0\linewidth]{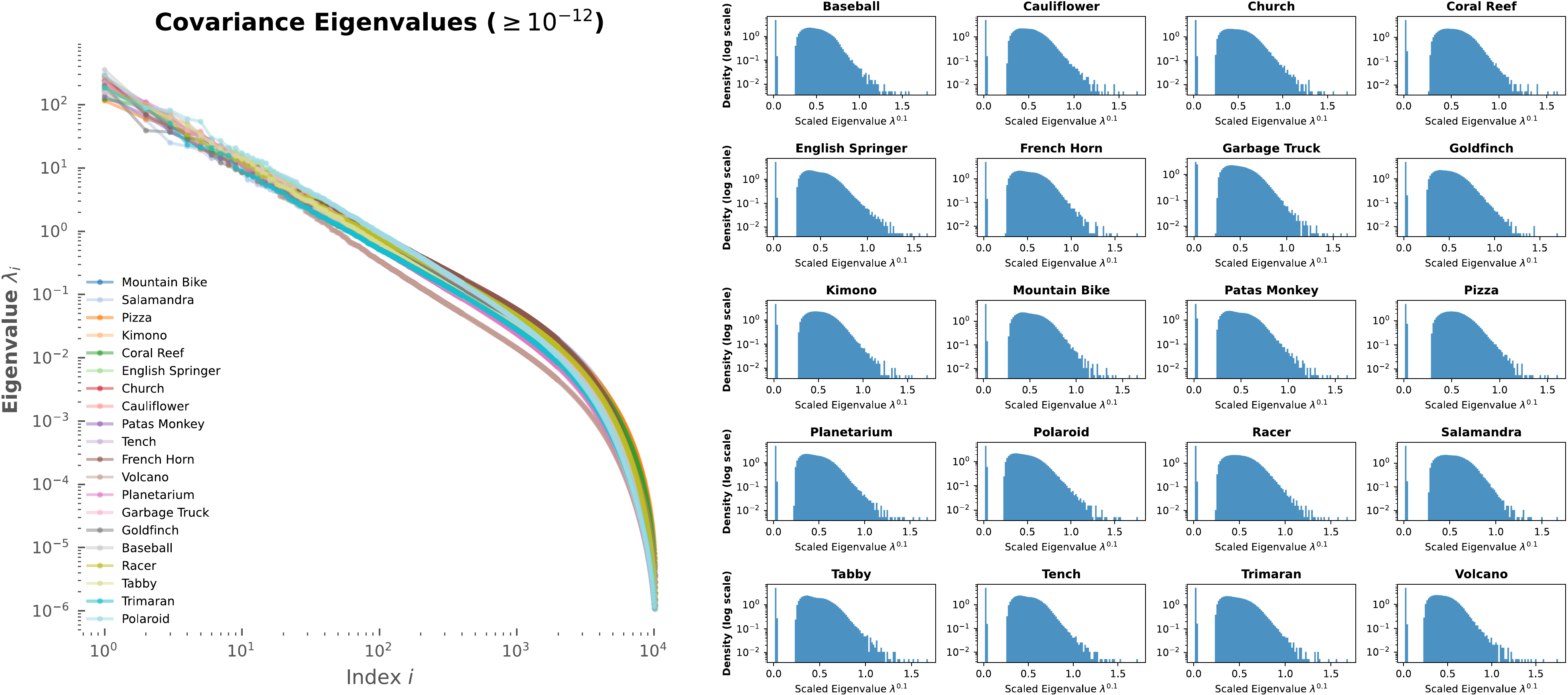}
    \caption{(A) All ordered eigenvalues ($\ge 10^{-12}$) of the covariance matrices for each class, shown log-log scale. (B) Full spectra of the covariance matrices of diffusion-generated images, computed over 10240 samples per class. Scaled by exponent $0.1$ for clearer presentation.}
    \label{fig:cov_spectra_more}
\end{figure}

\section{Representations of High-Resolution Latent Diffusion Samples} \label{section:edm2_gram}
Lastly, we investigate pre-trained classifier representations of high-resolution images generated from a latent diffusion model. We generated a dataset of $512 \times 512$ px images with deterministic sampling from EDM2 \cite{edm2} (\textbf{large}) using classifier-free guidance and guidance strength chosen to minimize Fréchet distance computed in the DINOv2 feature space \cite{dinov2}. We then resize to $256 \times 256$ px and apply the standard $224 \times 224$ px center crop before feeding to ResNets \cite{he_resnet} of various depths. This pre-processing is done to match the resolution that these ResNets were trained on. The representations are the output after global average pooling, before the final fully connected layer. They are of dimension 512 for Resnet18 and 2048 for Resnet50 and Resnet101.

\begin{figure*}[htbp]
    \centering
    \includegraphics[width=1.0\linewidth]{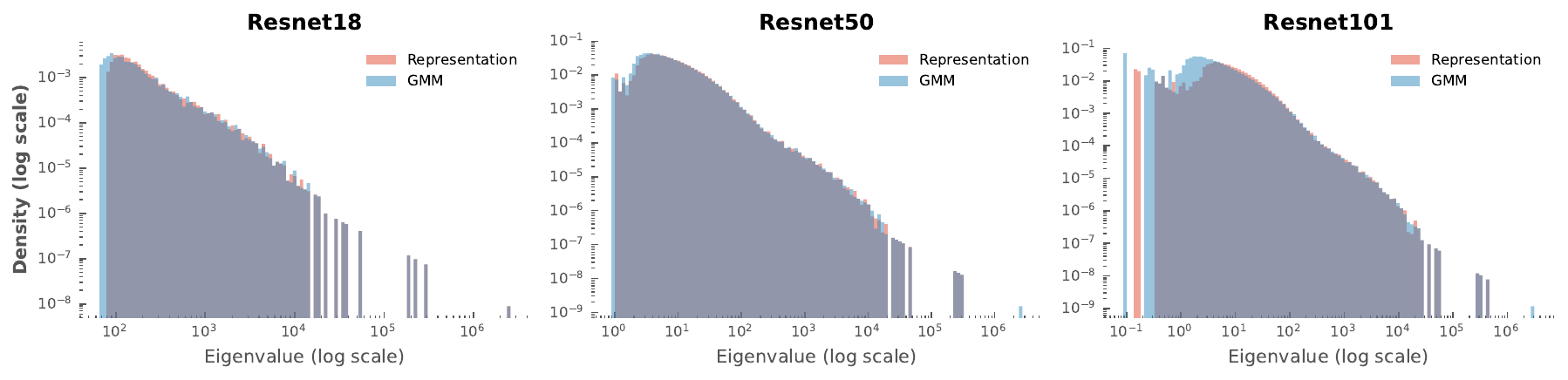}
    \caption{Spectra of Gram Matrices of ResNet representations of a 4-class mixture ({\small \texttt{Church, Tench, English Springer, French Horn}}) using 1350 images per class from EDM2 {\color{red} (Red)}, and for GMM fitted on those representations {\color{blue} (Blue)}.}
    \label{fig:representations_gramians}
\end{figure*}

As seen in Figure \ref{fig:representations_gramians}, the Gram matrices of the ResNet representations of diffusion images show a close to match to their GMM counterparts when viewing the eigenvalue spectrum. Moreover, we observed clear separability of the classes in the first few eigenvectors. This motivated us to train a logistic regression model on the top eigenvectors of these Gram matrices, and as shown in Figure \ref{fig:logistic_regression_representations}, the first 3-4 eigenvectors are all that are needed for near-perfect accuracy. We leave the question of how this scales with the number of classes for future work. We plot the mean and standard deviation envelope over 100 runs of logistic regression, each using a random split proportion of $0.8$ training samples from the 1350 representations per class in the mixture. The test set is fixed as a $0.2$ proportion of the representations of \textit{real} images.

\begin{figure}[htbp]
    \centering
    \includegraphics[width=0.3\linewidth]{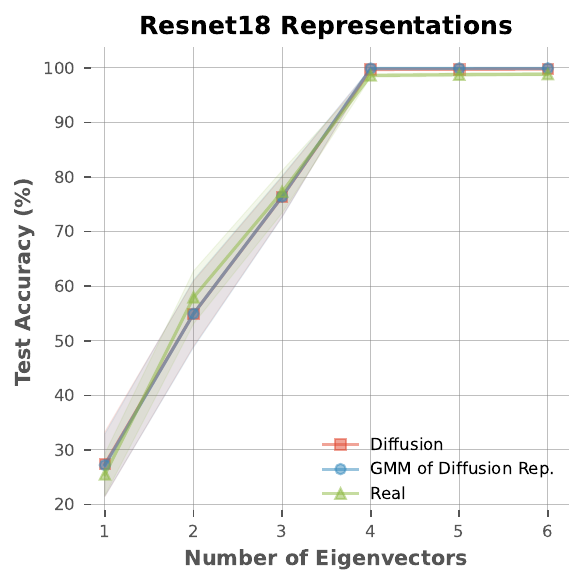}
    \includegraphics[width=0.3\linewidth]{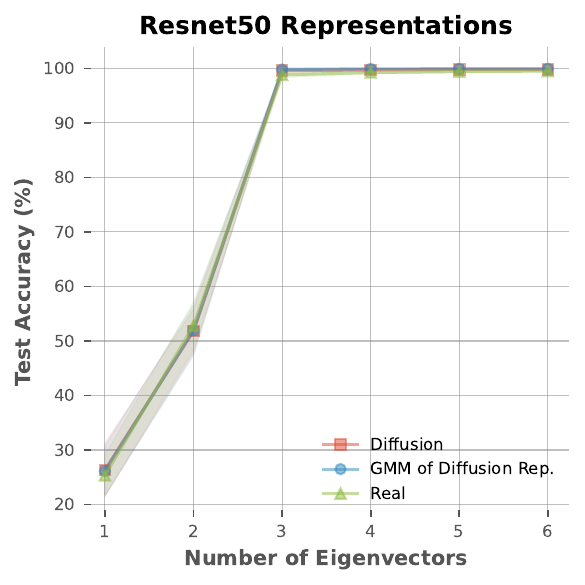}
    \includegraphics[width=0.3\linewidth]{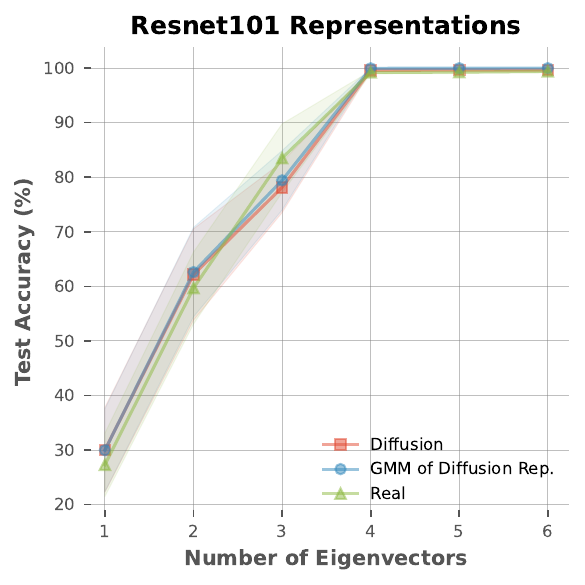}
    \caption{Logistic regression trained on ResNet representations of a 4-class mixture. For EDM2 representations {\color{red} (Red)}, GMM fit on EDM2 reps {\color{blue} (Blue)}, reps of real images {\color{green} (Green)}.}
    \label{fig:logistic_regression_representations}
\end{figure}

\begin{figure}[htbp]
    \centering
    \includegraphics[width=0.32\linewidth]{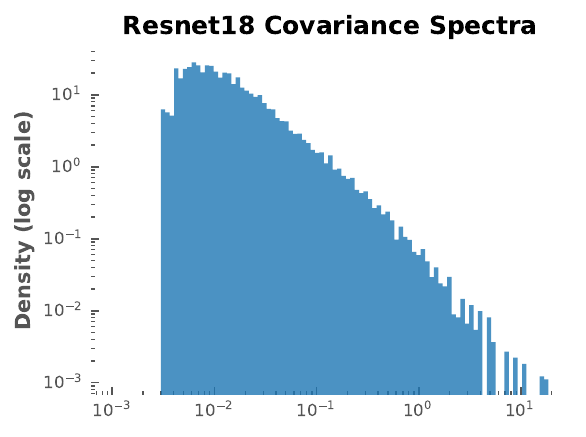}
    \includegraphics[width=0.32\linewidth]{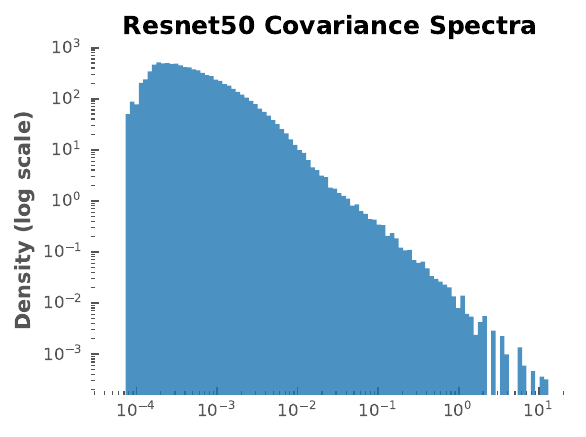}
    \includegraphics[width=0.32\linewidth]{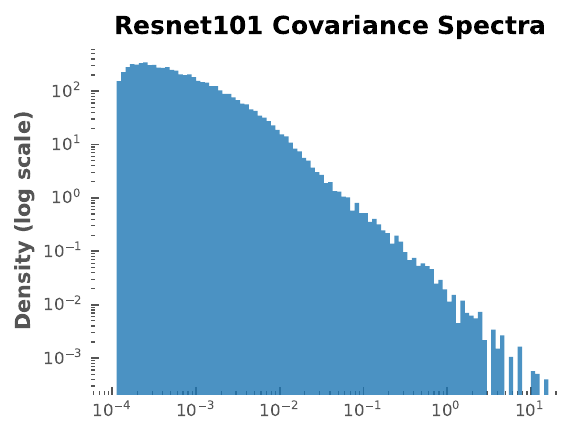}
    \caption{Eigenvalues of single class covariance matrix of ResNet reps of EDM2 images.}
    \label{fig:cov_spectra_representations}
\end{figure}

\begin{figure}[H]
    \centering
    \includegraphics[width=0.32\linewidth]{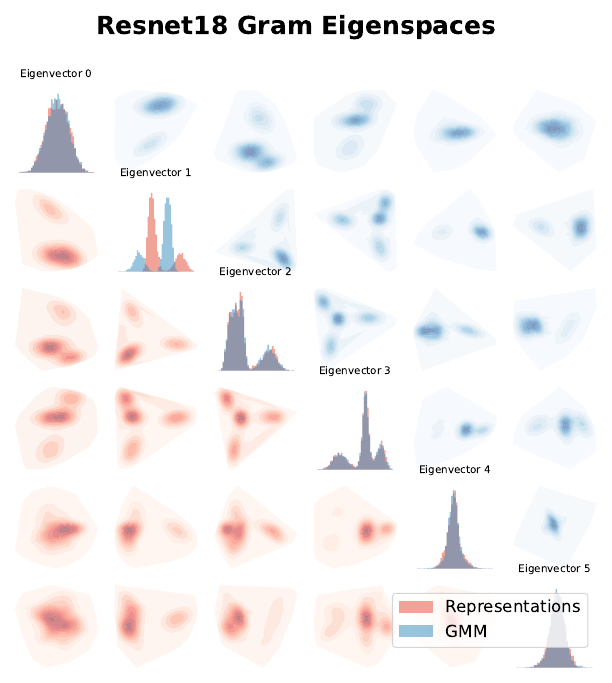}
    \includegraphics[width=0.32\linewidth]{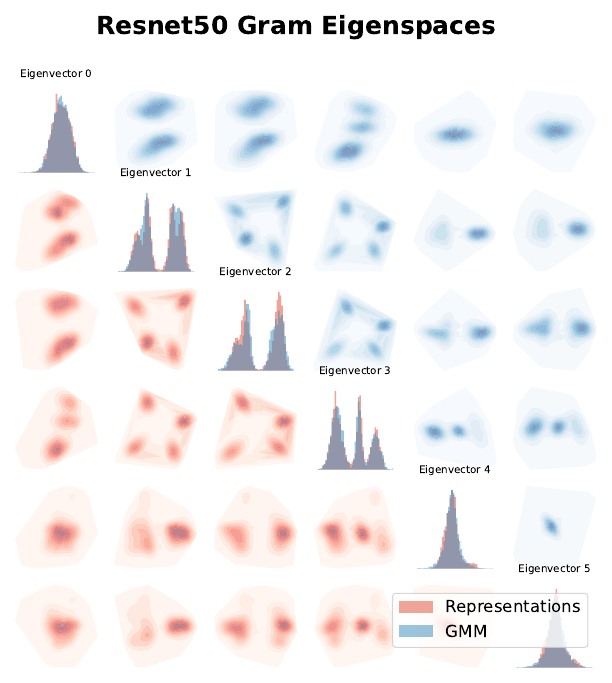}
    \includegraphics[width=0.32\linewidth]{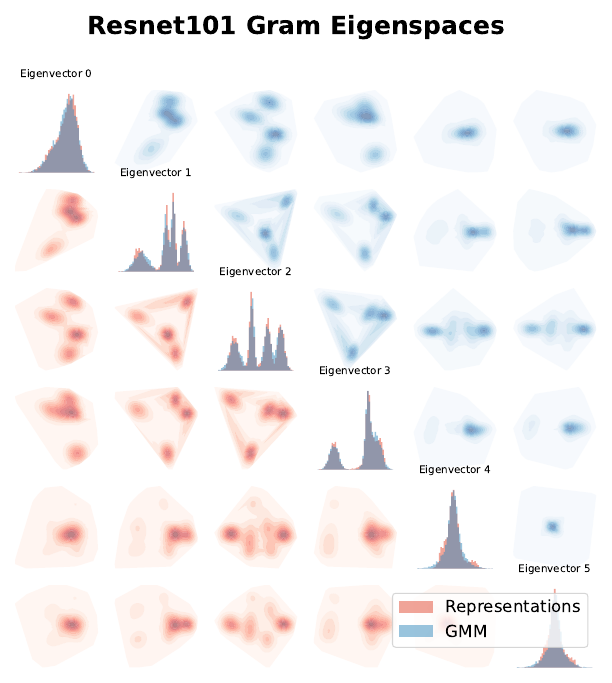}
    \caption{Top eigenspaces of Gram matrices of ResNet representations of 4-class mixture of EDM2 images. Corner plot of eigenvector $i$ vs. $j$ (Gaussian KDE) for representations of EDM2 {\color{red} (Red)} and for GMM fitted on representations {\color{blue} (Blue)}.}
    \label{fig:eigenspace_representations}
\end{figure}
We also plot the top eigenspaces of the Gram matrices for multiclass mixtures of these ResNet representations, compared against GMM samples fit on those representations. Figure \ref{fig:eigenspace_representations} demonstrates visually the match in the eigenvectors, and supplements the logistic regression experiments in our main body, where we trained logistic regression on the top $k$ eigenvectors of the Gram matrices and showed matching (with the GMM) test accuracies.

As a latent diffusion model, EDM2 \cite{karras2024analyzing} does diffusion in the latent space of a pre-trained variational autoencoder (VAE). We investigate the evolution of the norms of the latents through the deterministic sampling process, presented in Figure \ref{fig:edm2_norm_evolution}.

\begin{figure}[htbp]
    \centering
    \subfloat[\centering \small{Intermediate latents}]{{\includegraphics[width=0.5\linewidth]{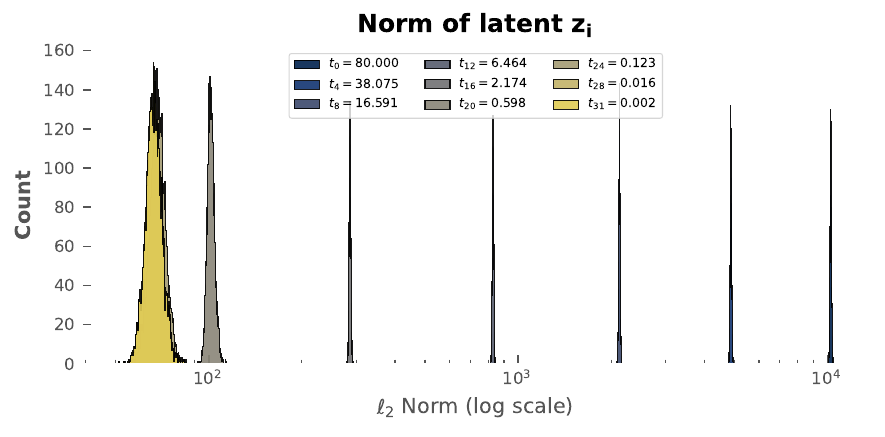} }}%
    \subfloat[\centering \small{Output of Denoiser on latents}]
    {{\includegraphics[width=0.5\linewidth]{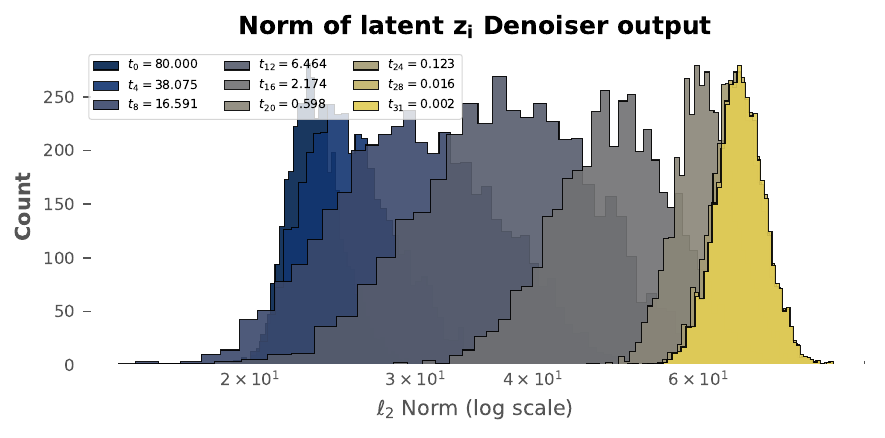} }}%
    \caption{Evolution of norms of latents through EDM2 deterministic sampling, for a single class. These are latents of dimension $4 \times 64 \times 64$ in the latent space of a pre-trained VAE.}
    \label{fig:edm2_norm_evolution}
\end{figure}

\section{Evolution of Norms of Pixels}
We also investigate the norms of individual pixels through the EDM sampling process. Note that Figure \ref{fig:pixel_norms} \textbf{c} is on a log-log scale, which cuts off negative values of the plotted standard deviation envelope at the low noise scales; indeed, at any step of the sampling process, there are pixels that increase in norm. But on average, the pixel norms are decreasing. 

\begin{figure}[htbp]
    \centering
    \subfloat[\centering \scriptsize{}{Distribution through sampling}]{{\includegraphics[width=0.32\linewidth]{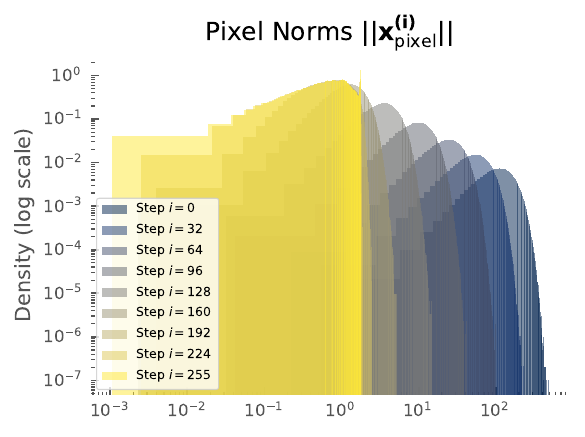} }}%
    \subfloat[\centering \scriptsize{}{Trajectories of pixel norms}]{{\includegraphics[width=0.32\linewidth]{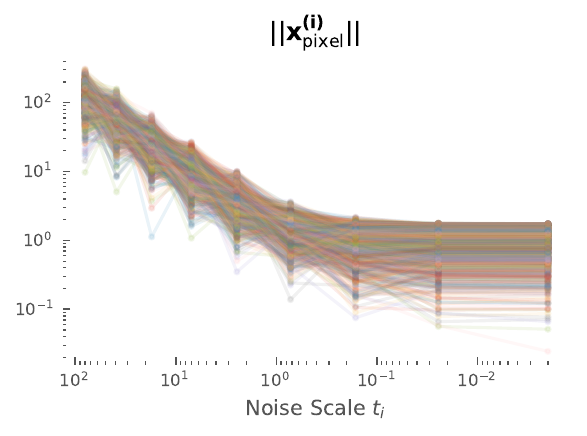} }}%
    \subfloat[\centering \scriptsize{}{Mean with Standard Deviation}]{{\includegraphics[width=0.32\linewidth]{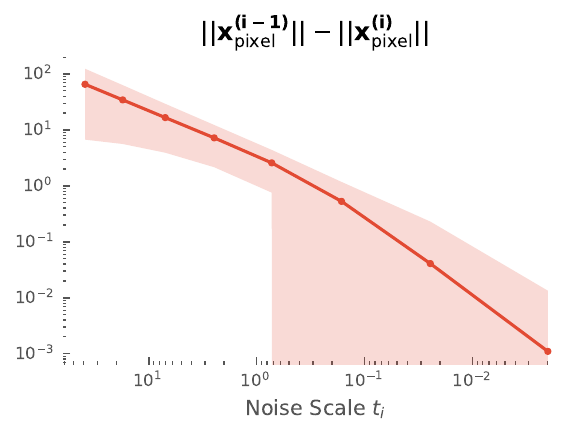} }}%
    \caption{\textbf{(a)} The distribution of the pixel norms for a single class, through the EDM sampling process. \textbf{(b)} The individual trajectories of the norms of 1000 randomly selected pixels at different noise scales of sampling. \textbf{(c)} The mean and standard deviation envelope of the difference in norms of pixels between sampling steps (Note that negative values are cut off, on this log-log scale.)}
    \label{fig:pixel_norms}%
\end{figure}

\section{Dataset Sample}
\begin{figure}[htbp]
    \centering
    \includegraphics[width=0.75\linewidth]{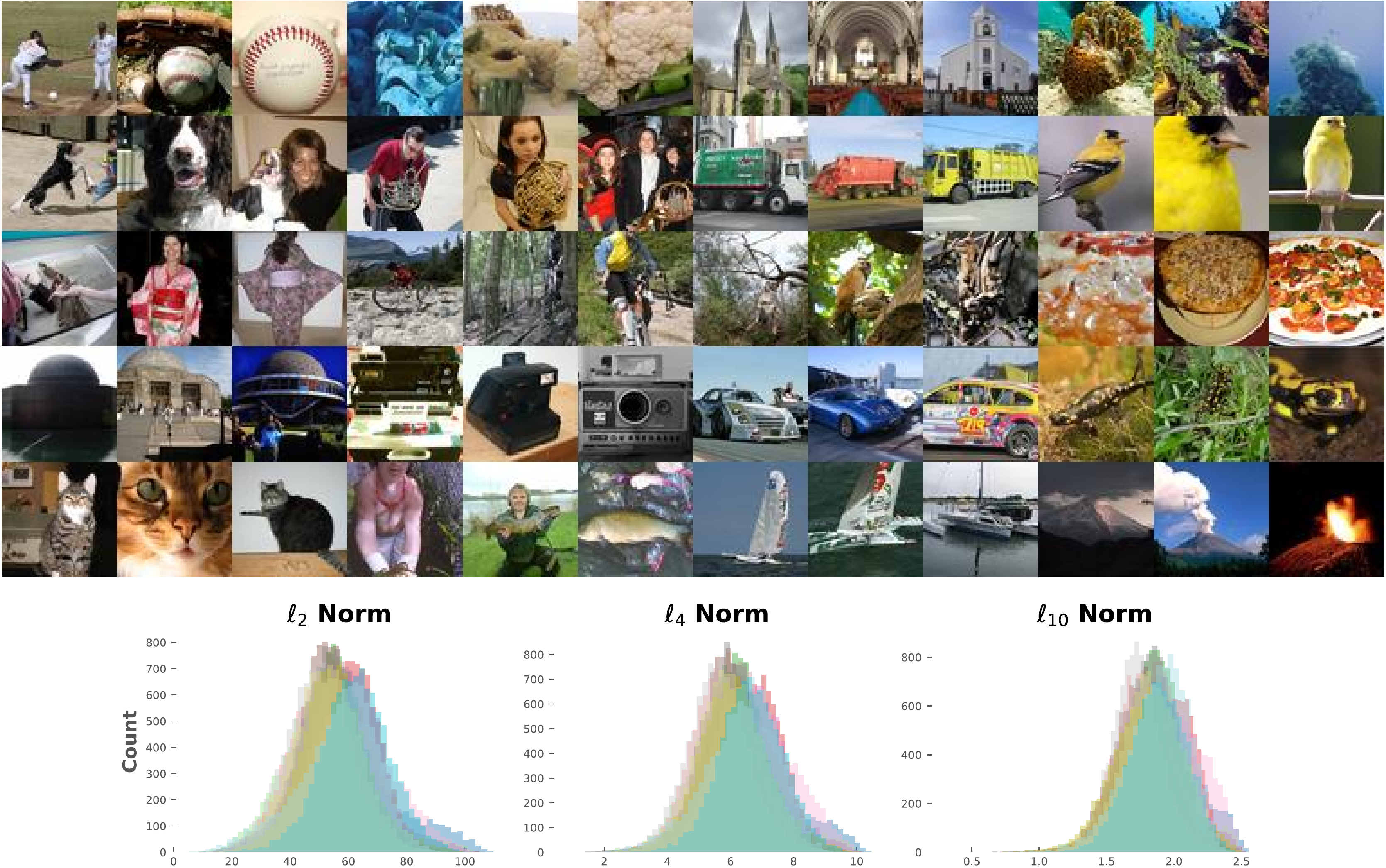}
    \caption{(A) Samples from conditional EDM diffusion model, trained on \texttt{Imagenet64}. (B) Distribution of $\ell_2, \ell_4$ and $\ell_{10}$ norms, computed over 10240 samples per class.}
    \label{fig:dataset-sample}
\end{figure}



\end{document}